%% file: main.tex
\documentclass{article} %
\usepackage{main,times}

\input{math_commands.tex}

\usepackage{hyperref}
\usepackage{url}
\usepackage{graphicx}

\usepackage{wrapfig}
\usepackage{color}
\usepackage{xcolor}
\usepackage{todonotes}
\usepackage{subfigure}
\usepackage{caption}

\usepackage[utf8]{inputenc} %
\usepackage[T1]{fontenc}    %
\usepackage{booktabs}       %
\usepackage{amsfonts}       %
\usepackage{nicefrac}       %
\usepackage{microtype}      %
\usepackage{amsmath, amsthm}
\usepackage{algorithmic}
\usepackage{algorithm}

\newtheorem{theorem}{Theorem}

\newtheorem{definition}{Definition}
\newtheorem{lemma}{Lemma}

\newtheorem{assumption}{Assumption}

\newtheoremstyle{TheoremNum}
    {\topsep}{\topsep}              %
    {\itshape}                      %
    {}                              %
    {\bfseries}                     %
    {.}                             %
    { }                             %
    {\thmname{#1}\thmnote{ \bfseries #3}}%
\theoremstyle{TheoremNum}
\newtheorem{theoremnum}{Theorem}

\usepackage[capitalise,nameinlink,noabbrev]{cleveref}

\crefname{assumption}{assumption}{assumptions}

\newcommand{\task}{\mathcal{E}}

\definecolor{light-green-3}{HTML}{DAEAD3}
\definecolor{light-red-3}{HTML}{F3CCCC}
\definecolor{light-blue-3}{HTML}{CFE1F3}

\title{Block Contextual MDPs for \\ Continual Learning}

\author{%
  Shagun Sodhani \\
  Facebook AI Research \\
   \And
   Franziska Meier \\
   Facebook AI Research \\
   \And
   Joelle Pineau \\
   Facebook AI Research \\
   McGill University \\
   \And 
   Amy Zhang \\
   Facebook AI Research \\
   UC Berkeley \\
   \texttt{amyzhang@fb.com} \\
}

\iclrfinalcopy %
\begin{document}

\maketitle

\begin{abstract}

In reinforcement learning (RL), when defining a Markov Decision Process (MDP), the environment dynamics is implicitly assumed to be stationary. This assumption of stationarity, while simplifying, can be unrealistic in many scenarios. In the continual reinforcement learning scenario, the sequence of tasks is another source of nonstationarity. In this work, we propose to examine this continual reinforcement learning setting through the \textit{block contextual MDP} (BC-MDP) framework, which enables us to relax the assumption of stationarity. This framework challenges RL algorithms to handle both nonstationarity and rich observation settings and, by additionally leveraging smoothness properties, enables us to study generalization bounds for this setting. Finally, we take inspiration from adaptive control to propose a novel algorithm that addresses the challenges introduced by this more realistic BC-MDP setting, allows for zero-shot adaptation at evaluation time, and achieves strong performance on several nonstationary environments.

\end{abstract}

\section{Introduction}

In the standard reinforcement learning (RL) regime, many limiting assumptions are made to make the problem setting tractable. A typical assumption is that the environment is stationary, i.e., the dynamics and reward do not change over time. However, most real-world settings -- from fluctuating traffic patterns to evolving user behaviors in digital marketing to robots operating in the real world -- do not conform to this assumption. In the more extreme cases, even the observation and action space can change over time. These setups are commonly grouped under the continual learning paradigm~\citep{ring1994continual, thrun1998lifelong, hadsell2020embracing} and non-stationarity is incorporated as a change in the task or environment distribution (that the agent is operating in). The ability to handle non-stationarity is a fundamental building block for developing continual learning agents~\citep{khetarpal2020towards}.

\begin{wrapfigure}{r}{0.3\textwidth}
    \begin{center}
        \includegraphics[width=0.22\textwidth]{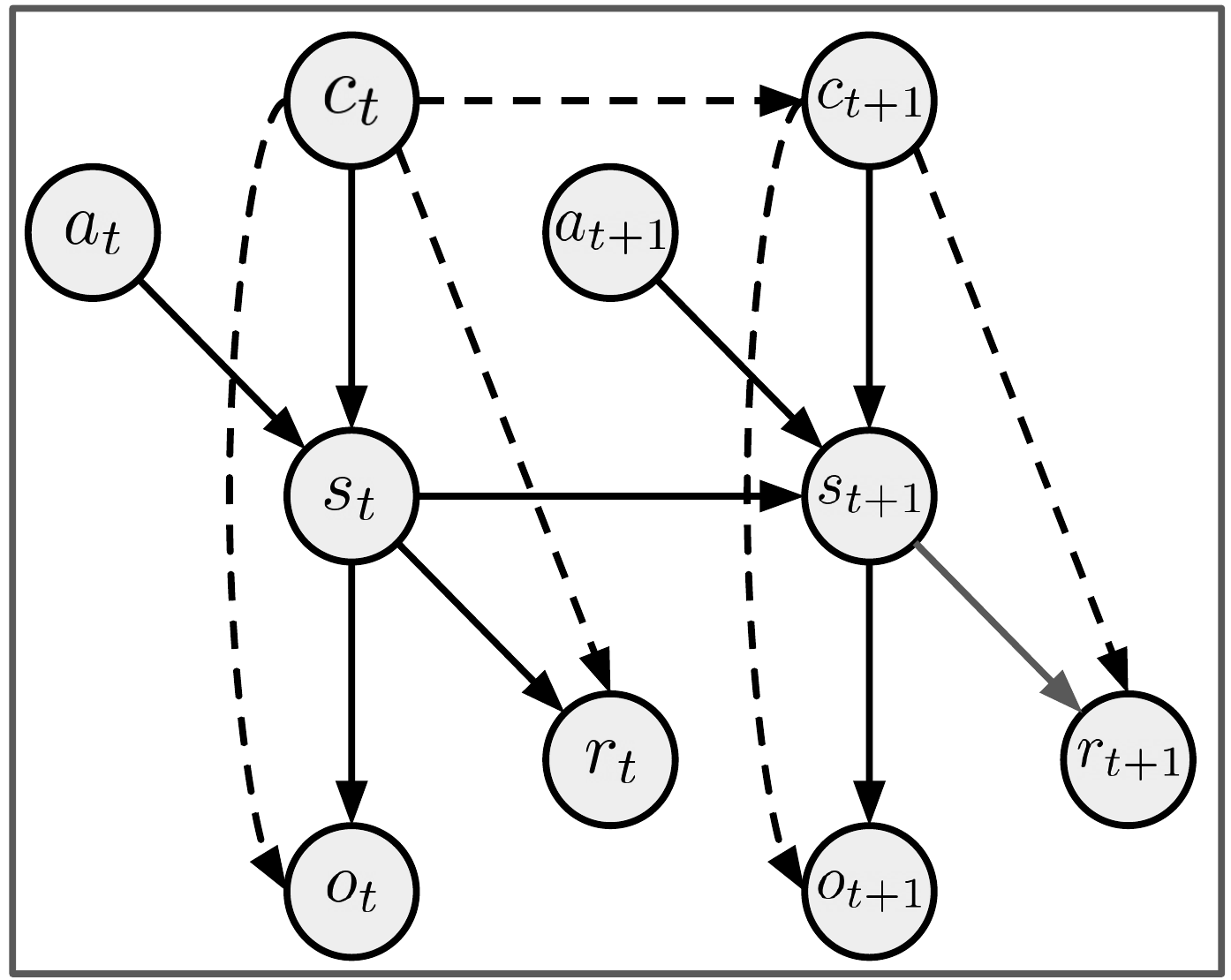}
    \end{center}
\caption{Graphical model of the BC-MDP setting.}
\label{fig:pgm}
\end{wrapfigure}

Real life settings present an additional challenge: we can not rely on access to (or knowledge of) an interpretable and compact (if not minimal) state space. Often, we only have access to a rich and high-dimensional observation space. For example, when driving a car on a wet street, we only have access to the ``view'' around us and not the friction coefficient between the car and the street. Hence, we must assume that observation contains irrelevant information (that could hinder generalization) and account for it when designing agents that could successfully operate in the nonstationary environments.

We propose to model this more realistic, rich observation, nonstationary setting as a Block Contextual MDP (BC-MDP) (shown in Fig.~\ref{fig:pgm}) by combining two common assumptions: (i) the~\textit{block assumption}~\citep{du2019pcid} that addresses rich observations with irrelevant features and (ii) the~\textit{contextual MDP}~\citep{hallak2015contextual} assumption - MDPs with different dynamics and rewards share a common structure and a context that can describe the variation across tasks. We introduce the Lipschitz Block Contextual MDP framework that leverages results connecting Lipschitz functions to generalization~\citep{xu2012robustness} and enables us to frame nonstationarity as a changing context over a family of stationary MDPs (thus modelling it as a contextual MDP). We propose a representation learning algorithm to enable the use of the current RL algorithms (that rely on the prototypical MDP setting) in nonstationary environments. It works by constructing a context space that is Lipschitz with respect to the changes in dynamics and reward of the nonstationary environment. We show, both theoretically and empirically, that the trained agent generalizes well to unseen contexts. We also provide value bounds based on this approximate abstraction which depend on some basic assumptions. 

This works is inspired from adaptive control~\citep{slotine_applied_1991}, a control method that continuously performs parameter identification to adapt to nonstationary dynamics of the system. Adaptive control generally considers the ``known unknowns,'' where the environment properties are known, but their values are unknown. We focus on the ``unknown unknowns'' setting, where the agent neither knows the environment property nor does it know its value in any task. Our setup is similar to meta-learning methods that ``learn to learn,'' but meta-learning techniques generally require finetuning or updates on the novel tasks~\citep{finn2017maml,rakelly2019pearl}. \textit{Our method does not need to perform parameter updates on the new tasks and can adapt in a zero-shot manner.} Since there are no parameter updates, our model does not suffer from catastrophic forgetting~\citep{mccloskey1989catastrophic}. This property is especially critical when designing continual learning agents that operate in the real world. Further, our model can be \textit{verified} after training and is guaranteed to stay true to that verification while also being capable of adapting zero-shot to new environments at test time. Intuitively, this follows from the observation that the agent's parameters are not updated when adapted to the unseen environments. We do not perform this type of formal verification, as current verification methods on neural networks only work for very small models, and are expensive to run~\citep{katz2019marabou}. We refer to our proposed method as \textbf{Ze}ro-shot adaptation to \textbf{U}nknown \textbf{S}ystems (ZeUS).

\paragraph{Contributions.}
We 1) introduce the Lipschitz Block Contextual MDP framework for the continual RL setting, 2) provide theoretical bounds on adaptation and generalization ability to unseen tasks within this framework utilizing Lipschitz properties, 3) propose an algorithm (ZeUS) to perform online inference of ``unknown unknowns'' to solve a family of tasks (without performing learning updates at test time) and ensure the prior Lipschitz properties hold, and 4) empirically verify the effectiveness of ZeUS on environments with nonstationary dynamics or reward functions.

\section{Related Work}
\label{sec:related_work}

Our work is related to four broad areas: (i) System Identification and Adaptive Control, (ii) Continual RL, (iii) Context Modeling and (iv) Meta RL and Multitask RL.

\textbf{System Identification and Adaptive Control}~\citep{zadeh1956identification, e0da8f25-d850-4a80-af21-e151cc28c4f4, swevers1997optimal, bhat2002computing, gevers2006system, LJUNG20101, van2012subspace, chiuso2019system, ajay2019combining, yu2017osi, zhu2017fast}. In this setup, the goal is to perform system identification of ``known unknowns,'' where the environment properties are known, but their values are unknown. Continuing with the previous example of driving a car on a wet street, in this setup, the agent knows that friction coefficient varies across tasks but does not know its value. By inferring the value from observed data, the agent can condition its policy (on the inferred value) to solve a given task. The conventional approaches alternate between system identification and control policy optimization. One limitation of this approach is that a good initial policy and system identifier are required for efficient training. We extend this setup to the ``unknown unknowns'' setting, where the agent neither knows the environment property nor does it know its value in any task.

Our work is related to the \textbf{Continual (or Lifelong) RL}~\citep{ring1994continual, 10.1145/2523813, abel2018lifelong, kaplanis2018continual, xu2018reinforced, aljundi2019online, javed2019meta, hadsell2020embracing}. In this setup, nonstationarity can manifest in two ways: (i)~\textit{Active nonstationarity} where the agent's actions may change the environment dynamics or action space (e.g., a cleaning robot tripping over the carpet), (ii)~\textit{Passive nonstationarity} where the environment dynamics may change irrespective of the agent's behavior (e.g., driving a car on a snow-covered road)~\citep{khetarpal2020towards}. Our work relates to the passive nonstationarity setup. Unlike~\citet{lopez2017gradient, chaudhry2019tiny, aljundi2019online,10.1162/neco_a_01246} which focus on challenges like catastrophic forgetting~\citep{mccloskey1989catastrophic}~\footnote{Since our model does not perform parameter updates when transferring to unseen tasks, it does not suffer from catastrophic forgetting.}, we focus on the ability to continually adapt (the policy) to unseen tasks~\citep{hadsell2020embracing}.~\citet{xie2020lilac} proposed LILAC that uses a probabilistic hierarchical latent variable model to learn a representation of the environment from current and past experiences and perform off-policy learning. A key distinction of our work is that we use task metrics to learn a context space and focus on generalization to unseen contexts.

Several works have focused on~\textbf{modeling the environment context} from high-level pixel observations~\citep{pathak2017curiosity, ebert2018visual, chen2018hardware, xu2019densephysnet}. This context (along with the observation) is fed as input to the policy to enable it to adapt to unseen dynamics (by implicitly capturing the dynamics parameters). T environment's context is encoded using a~\textit{context encoder} using the history of interactions. These approaches learn a single, global dynamics model conditioned on the output of the context encoder. Similar to these approaches, we also use a context encoder but introduce an additional loss to learn a context space with Lipschitz properties with respect to reward, and dynamics.~\citet{xian2021hyperdynamics} proposed using HyperNetworks~\citep{ha2016hypernetworks, chang2019principled, klocek2019hypernetwork, meyerson2019modular} that use the context to generate the weights of the \textit{expert} dynamics model.

Other works on structured MDPs, that leverage the Lipschitz properties, include ~\cite{modi2018markov} that assumes that the given contextual MDP is smooth and that the distance metric and Lipschitz constants are known. In contrast, we propose a method that constructs a new smooth contextual MDP, with bounds on downstream behavior based on the approximate-ness of the new contextual MDP.~\cite{modi2020no} propose RL algorithms with lower bounds on regret but assume that the context is known and linear with respect to the MDP parameters. In contrast, we do not assume access to the context at train or test time or linearity with respect to MDP parameters. 

~\textbf{Meta-reinforcement learning} aims to ``meta-learn'' how to solve new tasks efficiently~\citep{finn2017maml,clavera2018learning,rakelly2019pearl, Zhao2020meld}. Optimization-based meta-RL methods~\citep{finn2017maml, mishra2017simple, zintgraf2019fast} require updating model parameters for each task and therefore suffer from catastrophic forgetting in the continual learning setting. Context-based meta-RL methods perform online adaptation given a context representation (generally modeled as the hidden representations of a RNN~\citep{nagabandi2018deep}). The hope is that the model would (i) adapt to the given context and (ii) correctly infer the next state. However, follow-up work~\citep{lee2020cadm} suggests that it is better to disentangle the two tasks by learning a context encoder (for adaption) and a context-conditioned dynamics model (for inferring the next state).~\citet{lee2020cadm} also introduced additional loss terms when training the agent. However, their objective is to encourage the context encoding to be useful for predicting both forward (next state) and backward (previous state) dynamics while being temporally consistent, while our objective is to learn a context space with Lipschitz properties with respect to reward and dynamics. Other works have proposed modeling meta-RL as task inference~\cite{humplik2019meta, kamienny2020learning} but these works generally assume access to some~\textit{privileged information} (like~\textit{task-id}) during training.

Our work is also related to the general problem of training a policy on Partially Observable Markov Decision Processes (POMDPs)~\citep{kaelbling1998planning, igl2018deep, zhang2019causalstates, han2019variational, hafner2019learning} that capture both nonstationarity and rich observation settings. Our experiments are performed in the POMDP setup where we train the agent using pixel observations and do not have access to a compact state-space representation. However, we focus on a specific class of POMDPs --- the contextual MDP with hidden context, which enables us to obtain strong generalization performance to new environments. Finally, we discuss additional related works in multi-task RL, transfer learning, and MDP metrics in \cref{app:related_work}. 

\section{Background \& Notation}
A \textbf{Markov Decision Process} (MDP)~\citep{Bellman1957,puterman1995markov} is defined by a tuple $\langle {\cal S}, {\cal A}, R, T, \gamma \rangle $, where ${\cal S}$ is the set of states, ${\cal A}$ is the  set of actions, $R: {\cal S} \times {\cal A}\rightarrow \mathbb{R}$ is the reward function, $T:{\cal S} \times {\cal A} \rightarrow Dist({\cal S})$ is the environment transition probability function, and $\gamma \in [0,1)$ is the discount factor. At each time step, the learning agent perceives a state $s_t \in {\cal S}$, takes an action $a_t \in {\cal A}$ drawn from a policy $\pi : {\cal S} \times {\cal A} \rightarrow [0,1]$, and with probability $T(s_{t+1}|s_t,a_t)$ enters next state $s_{t+1}$, receiving a numerical reward $R_{t+1}$ from the environment. The value function of policy $\pi$ is defined as: $V_\pi(s) = E_\pi[\sum_{t=0}^{\infty} \gamma^{t} R_{t+1} | S_0 = s]$. The optimal value function $V^{*}$ is the maximum value function over the class of stationary policies.

\textbf{Contextual Markov Decision Processes} were first proposed by \citet{hallak2015contextual} as an augmented form of Markov Decision Processes that utilize \textit{side information} as a form of context, similar to contextual bandits. For example, the friction coefficient of a surface for a robot sliding objects across a table is a form of context variable that affects the environment dynamics, or user information like age and gender are context variables that influence their movie preferences.

\begin{definition}[Contextual Markov Decision Process]
A contextual Markov decision process (CMDP) is defined by tuple $\langle \mathcal{C}, \mathcal{S}, \mathcal{A}, \mathcal{M}\rangle$ where $\mathcal{C}$ is the context space, $\mathcal{S}$ is the state space, $\mathcal{A}$ is the action space. $\mathcal{M}$ is a function which maps a context $c\in\mathcal{C}$ to MDP parameters $\mathcal{M}(c)=\{R^c, T^c\}$.
\end{definition}

However, in the real world, we typically operate in a ``rich observation'' setting where we do not have access to a compressed state representation and the learning agent has to learn a mapping from the observation to the state. This additional relaxation of the original CMDP definition as a form of Block MDP~\citep{du2019pcid} was previously introduced in \citet{sodhani2021multi} for the multi-task setting where the agent focuses on a subset of the whole space for a specific task, which we again present here for clarity:

\begin{definition}[Block Contextual Markov Decision Process~\citep{sodhani2021multi}]
A block contextual Markov decision process (BC-MDP) (Fig.~\ref{fig:pgm}) is defined by tuple $\langle \mathcal{C}, \mathcal{S}, \mathcal{O}, \mathcal{A}, \mathcal{M}\rangle$ where $\mathcal{C}$ is the context space, $\mathcal{S}$ is the state space, $\mathcal{O}$ is the observation space, $\mathcal{A}$ is the action space. $\mathcal{M}$ is a function which maps a context $c\in\mathcal{C}$ to MDP parameters and observation space $\mathcal{M}(c)=\{R^c, T^c, \mathcal{O}^c\}$.
\end{definition}

Consider a robot moving around in a warehouse and performing different tasks. Rather than specifying an observation space that covers the robot's lifelong trajectory, it is much more practical to have the observation space change as its location and attention change, as it would with an attached camera. We can still keep the assumption of full observability because the robot can have full information required to solve the current task, e.g. with frame stacking to capture velocity and acceleration information. The continual learning setting differs from sequential multi-task learning as there is no delineation of tasks when $c$ changes, causing nonstationarity in the environment. We make an additional assumption that the change in $c$ is smooth over time and the BC-MDP itself is smooth, as shown in \cref{def:smoothness}.

We now define a Lipschitz MDP for the MDP family we are concerned with.
\begin{definition}[Lipschitz Block Contextual MDP]
\label{def:smoothness}
Given a BC-MDP $\langle \mathcal{C}, \mathcal{S}, \mathcal{O}, \mathcal{A}, \mathcal{M}\rangle$ and a distance metric  $d(\cdot,\cdot)$ over context space, if for any two contexts $c_1,c_2\in\mathcal{C}$, we have the following constraints,
\begin{align*}
    \forall(s, a), W(T^{c_1}(s,a), T^{c_2}(s,a))&\leq L_p d(c_1,c_2), \\
    \forall(s, a), \|R^{c_1}(s,a)-R^{c_2}(s,a)\|&\leq L_r d(c_1,c_2),
\end{align*}
then the BC-MDP is referred to as a Lipschitz BC-MDP with smoothness parameters $L_p$ and $L_r$.
\end{definition}
Here $W$ denotes the Wasserstein distance. Note that \cref{def:smoothness} is not a limiting assumption because we do not assume access to the context variables $c_1$ and $c_2$, and they can therefore be chosen so that the Lipschitz condition is always satisfied. In this work, we focus on a method for learning a context space that satisfies the above property.

\section{Generalization Properties of Lipschitz BC-MDPs}
\label{sec:theory}
The key idea behind the proposed method (presented in full in~\cref{sec::zero-shot-adaptation-to-unknown-systems}) is to construct a context space $\mathcal{C}$ with Lipschitz properties with respect to dynamics and reward, and therefore, optimal value functions across tasks. In this section, we show how this Lipschitz property aids generalization. The following results hold true for any given observation (or state) space and are not unique to Block MDPs, so we use notation with respect to states $s\in\mathcal{S}$ without loss of generality. Since we do not have access to the true context space, in~\cref{sec::zero-shot-adaptation-to-unknown-systems}, we describe how to learn a context space with the desired characteristics. 

In order to construct a context space that is Lipschitz with respect to tasks, notably the optimal value functions across tasks, we turn to metrics based on state abstractions. 
Based on established results on distance metrics over states (see \cref{app:assumptions}), we can define a task distance metric for the continual RL setting. 
\begin{definition}[Task Metric]
\label{def:task_metric}
Given two tasks sampled from a BC-MDP, identified by contexts $c_i$ \& $c_j$, 
\begin{equation}
    d_{\text{task}}(c_i,c_j):=\max_{s,a\in \{S,A\}}\bigg[\big|R^{c_i}(s,a)-R^{c_j}(s,a)\big| + W(d_{\text{task}})\big(T^{c_i}(s, a),T^{c_j}(s,a)\big)\bigg],
\end{equation}
where $W(d_{\text{task}})$ is the Wasserstein distance between transition probability distributions.  

\end{definition}
We can now show that the dynamics, reward, and optimal value function are all also Lipschitz with respect to $d_{\text{task}}$. The first two are clear results from \cref{def:task_metric}. 
\begin{theorem}[$V_c^*$ is Lipschitz with respect to $d_{\text{task}}$]
\label{thm:lipschitz_bcmdp}
Let $V^*$ be the optimal, universal value function for a given discount factor $\gamma$ and context space $\mathcal{C}$. Then $V^*$ is Lipschitz continuous with respect to $d_{\text{task}}$ with Lipschitz constant $\frac{1}{1-\gamma}$ for any $s\in\mathcal{S}$,
$$|V^*(s, c) - V^*(s, c')|\leq \frac{1}{1-\gamma}d_{\text{task}}(c,c').$$
\end{theorem}
The proof can be found in~\cref{app:thms}.
Applying~\cref{thm:lipschitz_bcmdp} to a continual RL setting requires the context be identifiable from a limited number of interactions with the environment. This warrants the following assumption:

\begin{assumption}[Identifiability]
\label{ass:identifiability}
Let $k$ be some constant number of steps the agent takes in a new environment with context $c$. There exists an $\epsilon_c > 0$ such that a context encoder $\psi$ can take those transition tuples $(s_i,a_i,s_i',r_i),i\in\{1,...,k\}$ and output a predicted context $\hat{c}$ that is $\epsilon_c$-close to $c$.
\end{assumption}
There are two key assumptions wrapped up in Assumption~\ref{ass:identifiability}. The first is that the new environment is uniquely identifiable from $k$ transitions, and the second is that we have a context encoder that can approximately infer that context. In practice, we use neural networks for modeling $\psi$ and verify that neural networks can indeed learn to infer the context, as shown in~\cref{table::assumption_identifiability} in \cref{appendix:setup}.

Why do we care about the Lipschitz property?  \citet{xu2012robustness} established that Lipschitz continuous functions are robust, i.e. the gap between test and training error is bounded. This result is only useful when the problem space is Lipschitz, which is often not the case in RL. However, we have shown that any BC-MDP is Lipschitz continuous with respect to metric $d_\text{task}$. We now define a general supervised learning setup to bound the error of learning dynamics and reward models. The following result requires that the data-collecting policy is ergodic, i.e. a Doeblin Markovian chain~\citep{doob1953stochasticprocesses,meyn1993markov}, defined as follows.
\begin{definition}[Doeblin chain]
A Markov chain $\{s_i\}_{i=1}^\infty$ on a state space $\mathcal{S}$ is a Doeblin chain (with $\alpha$ and $t$) if there exists a probability measure $\rho$ on $\mathcal{S}$, $\alpha > 0$, an integer $t\geq 1$ such that
\begin{equation*}
    P(s_t\in H|s_0=s)\geq \alpha \rho(H); \quad \forall \text{ measurable } H\subseteq \mathcal{S}; \forall s\in \mathcal{S}.
\end{equation*}
\end{definition}
Let $\hat{\mathcal{L}}(\cdot)$ denote expected error and $\mathcal{L}_{\text{emp}}(\cdot)$ denote training error of an algorithm $\mathcal{A}$ on training data $\mathbf{s}=\{s_1,...,s_n\}$ and evaluated on points $z\in\mathcal{Z}$ sampled from distribution $\mu$:
\begin{equation*}
    \hat{\mathcal{L}}(\cdot):= \mathbb{E}_{z\sim \mu}\mathcal{L}(\mathcal{A}_{\mathbf{s}},z);  \quad \mathcal{L}_{\text{emp}}(\cdot):=\frac{1}{n}\sum_{s_i\in\mathbf{s}}\mathcal{L}(\mathcal{A}_\mathbf{s},s_i).
\end{equation*}

Here, $\mathcal{A}_{\mathbf{s}}$ denotes the instantiation of the learned algorithm trained on  data $\mathbf{s}$ whereas $\mathcal{A}$ refers to the general learning algorithm. We can now bound the generalization gap, the difference between expected and training error using a result from \citet{xu2012robustness}.
\begin{theorem}[Generalization via Lipschitz Continuity~\citep{xu2012robustness}]
\label{thm:robustness}
If a learning algorithm $\mathcal{A}$ is $\frac{1}{1-\gamma}$-Lipschitz and the training data $\mathbf{s}=\{s_1,...,s_n\}$ are the first $n$ outputs of a Doeblin chain with constants $\alpha, T$, then for any $\delta>0$ with probability at least $1-\delta$, 
\begin{equation*}
    \big|\hat{\mathcal{L}}(\mathcal{A}_\mathbf{s}) - \mathcal{L}_{\text{emp}}(\mathcal{A}_\mathbf{s})\big|\leq \frac{\epsilon}{1-\gamma} + M\bigg(\frac{8t^2(K\ln 2+\ln (1/\delta))}{\alpha^2n}\bigg)^{1/4}.
\end{equation*}
\end{theorem}
$K$ denotes the $\epsilon$-covering number of the state space. $\epsilon$ controls the granularity at which we discretize, or partition, that space. If $\epsilon$ is larger, $K$ is smaller. $M$ is a scalar that uniformly upper-bounds the loss $\mathcal{L}$. Once we learn a smooth context space, this result bounds the generalization error of supervised learning problems like learned dynamics and reward models. These learned models allow us to construct a new MDP that is $\epsilon_R,\epsilon_T,\epsilon_{c}$-close to the original. We can now show how this error propagates when learning a policy. 

\begin{theorem}[Generalization Bound]
\label{thm:gen_bound}
Without loss of generality we assume all tasks in a given BC-MDP family have reward bounded in $[0,1]$. Given two tasks $\mathcal{M}_{c_i}$ and $\mathcal{M}_{c_j}$, we can bound the difference in $Q^\pi$ between the two MDPs for a given policy $\pi$ learned under an $\epsilon_R,\epsilon_T,\epsilon_{c_i}$-approximate abstraction of $\mathcal{M}_{c_i}$ and applied to $\mathcal{M}_{c_j}$,
\begin{equation*}
\big\| Q^*_{\mathcal{M}_{c_j}} - [Q^*_{\bar{\mathcal{M}}_{\hat{c}_i}}]_{\mathcal{M}_{c_j}}\big\|_\infty \leq \epsilon_R + \gamma \big(\epsilon_T + \epsilon_{c_i} + \|c_i - c_j\|_1\big) \frac{1}{2(1-\gamma)}.
\end{equation*}
\end{theorem}
Proof in~\cref{app:thms}.~\cref{thm:gen_bound} shows that if we learn an $\epsilon$-optimal context-conditioned policy for task with context $c_i$ and encounter a new context $c_j$ at evaluation time where $c_j$ is close to $c_i$, then the context-conditioned policy will be $\epsilon$-optimal for the new task by leveraging the Lipschitz property.
While these results clearly do not scale well with the dimensionality of the state space and discount factor $\gamma$, it shows that representation learning is a viable approach to developing robust world models (\cref{thm:robustness}), which translates to tighter bounds on the suboptimality of learned $Q$ functions (\cref{thm:gen_bound}).

\section{Zero-shot Adaptation to Unknown Systems}
\label{sec::zero-shot-adaptation-to-unknown-systems}
Based on the findings in~\cref{sec:theory}, we can improve generalization by constructing a context space that is Lipschitz with respect to the changes in dynamics and reward of the nonstationary environment. In practice, computing the maximum Wasserstein distance over the entire state-action space is computationally infeasible. We relax this requirement by taking the expectation over Wasserstein distance with respect to the marginal state distribution of the behavior policy. This leads us to a representation learning objective that leverages this relaxed version of the task metric in~\cref{def:task_metric}:
\begin{align}
\small
\label{eq:total_loss}
\mathcal{L}(\phi,\psi,T,R) &= MSE\underbrace{\bigg(\Big|\Big|\psi(H_1)-\psi(H_2)\Big|\Big|_2, {\color{red}d}(c_1,c_2)\bigg)}_{\text{context loss}} 
+ MSE\underbrace{\bigg(T(\phi(o^{c_1}_t),a^{c_1}_t,\psi(H_1)), {\color{red}\phi}(o^{c_1}_{t+1})\bigg) 
}_{\text{Dynamics loss}}  \nonumber \\ 
& + MSE\underbrace{\bigg(R(\phi(o^{c_1}_t),a^{c_1}_t,\psi(H_1)), r^{c_1}_{t+1}\bigg)
}_{\text{Reward loss}}.
\end{align}
where {\color{red}red} indicates stopped gradients. $H_1:=\{o^{c_1}_t,a_t,r_t, o^{c_1}_{t+1},...\}$ and $H_2:=\{o^{c_2}_t,a_t, r_t, o^{c_2}_{t+1}, ...\}$ are transition sequences from two environments with contexts $c_1$ and $c_2$ respectively. During training, the transitions are uniformly sampled from a replay buffer. We do not require access to the true context for computing $d(c_1,c_2)$ (in~\cref{eq:total_loss}) as we can approximate $d(c_1,c_2)$ using~\cref{def:smoothness}. Specifically, we train a transition dynamics model and a reward model (via supervised learning) and use their output to approximate $d(c_1,c_2)$.
In practice, we scale the context learning error, our task metric loss, using a scalar value denoted as $\alpha_{\psi}$.

We describe the architecture of ZeUS in~\cref{fig:main}. We have an observation encoder $\phi$ that encodes the pixel-observations into real-valued vectors. A buffer of interaction-history is maintained for computing the context. The context encoder first encodes the individual state-action transition pairs and then aggregates the representations using standard operations:~\textit{sum},~\textit{mean},~\textit{concate},~\textit{product},~\textit{min} and~\textit{max}~\footnote{We experiment with these aggregation operators for all the baselines and not just ZeUS.}. All the components are instantiated using feedforward networks. During inference, assume that the agent is operating in some environment denoted by (latent) context $c_1$. At time $t$, the agent gets an observation $o_t^{c_1}$ which is encoded into $s_t^{c_1}:=\phi(o_t^{c_1})$\footnote{We overload notation here since the true state space is latent.}. The context encoder $\psi$ encodes the last $k$ interactions (denoted as $H_1$) into a context encoding $c_1:=\psi(H_1)$\footnote{We again overload notation here since the true context space is also latent.}. The observation and context encodings are concatenated and fed to the policy to get the action. 

During training, we sample a batch of interaction sequences from the buffer. For sake of exposition, we assume that we sample only 2 sequences $H_1$ and $H_2$. Similar to the inference pipeline, we compute $\phi(o_t^{c_1}), \psi(H_1), \phi(o_t^{c_2})$ and $\psi(H_2)$ and the loss (\cref{eq:total_loss}). We highlight that the algorithm does not know if the two (sampled) interactions correspond to the same context or not. Hence, in a small percentage of cases, $H_1$ and $H_2$ could correspond to the same context and the context loss will be equal to $0$. For implementing the loss in equation~\cref{eq:total_loss}, we do not need access to the true context as the distance between the contexts can be approximated using the learned transition and reward models using~\cref{def:smoothness}. The pseudo-code is provided in~\cref{alg:algorithm} (\cref{appendix:implementation_details}). Since ZeUS is a representation learning algorithm, it must be paired with a policy optimization algorithm for end-to-end training. In the scope of this work, we use Soft Actor-Critic with auto-encoder loss (SAC-AE, \cite{yarats2019improving}), though ZeUS can be used with any policy optimization algorithm.
    
\begin{figure}
    \centering
    \includegraphics[width=0.8\textwidth]{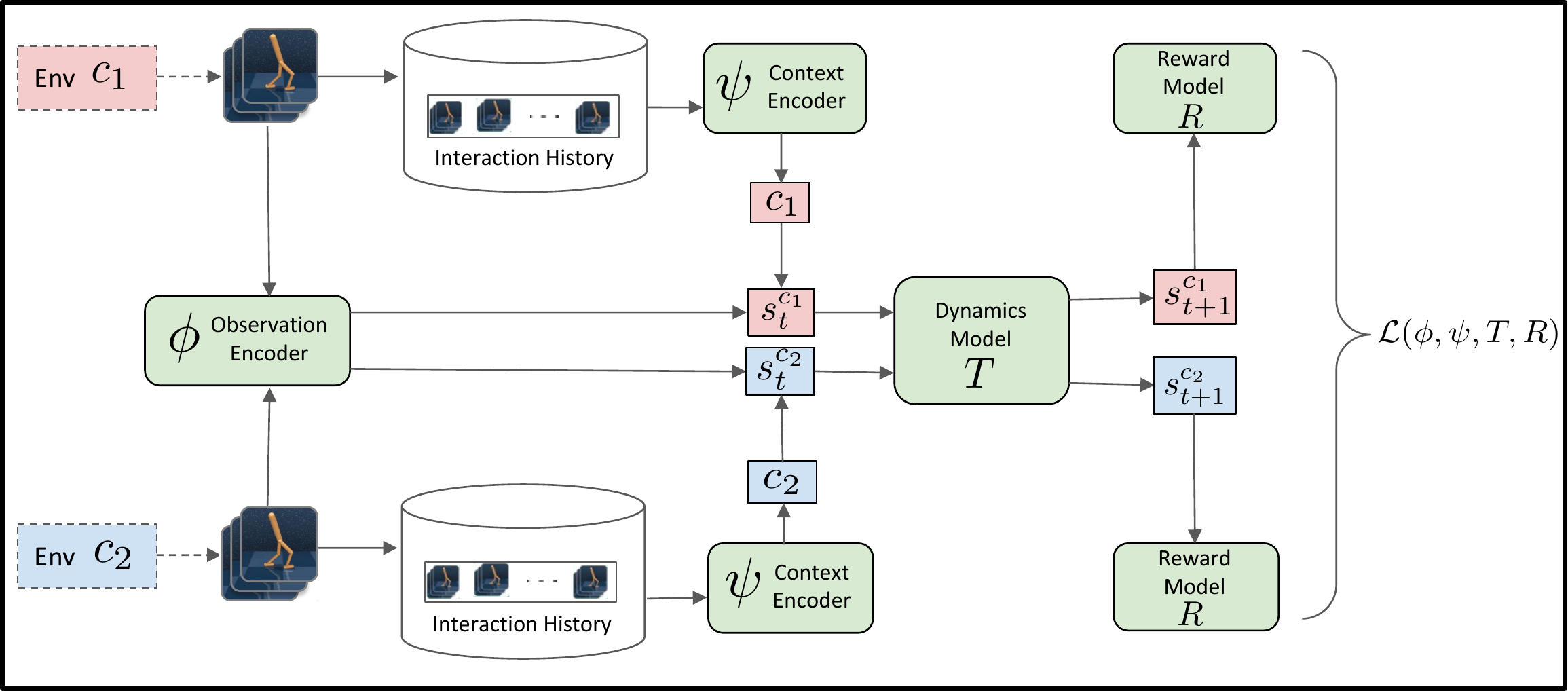}

    \caption{Proposed ZeUS algorithm. The components shown in green (i.e. observation encoder, context encoder, dynamics model and reward model) are shared across tasks. Components/representations in red or blue belong to separate tasks.}
    \label{fig:main}
\end{figure}

\section{Experiments}

We design our experiments to answer the following questions:
\begin{enumerate}
    \item How well does ZeUS perform when training over a family of tasks with varying dynamics? (see \cref{fig:cl-mujoco-all} in Appendix) %
    \item Can ZeUS adapt and generalize to unseen environments (with novel dynamics or reward) without performing any gradient updates? (see \cref{fig:cl-mujoco-all-eval_extrapolation} and \cref{fig:cl-meld-all-extrapolation})
    \item Can ZeUS learning meaningful context representations when training over a family of tasks with varying dynamics? (see \cref{fig:task_distance_zeus})
\end{enumerate}

\subsection{Setup}
\label{sec:setup}
Similar to the setups from~\citet{zhou2018environment, lee2020cadm, zhang2021hipbmdp}, we start with standard RL environments and extend them by modifying parameters that affect the dynamics (e.g. the friction coefficient between the agent and the ground) or the reward (e.g. target velocity) such that they exhibit the challenging nonstationarity and rich-observation conditions of our BC-MDP setting. For varying the transition dynamics, we use the following Mujoco~\citep{todorov2012mujoco}\footnote{License related information available at: https://www.roboti.us/license.html} based environments from the DM Control Suite~\citep{deepmindcontrolsuite2018}: Cheetah-Run-v0 (vary the length of the torso of the cheetah), Walker-Walk-v0 (vary the friction coefficient between the walker and the ground), Walker-Walk-v1 (vary the length of the foot of the walker) and Finger-Spin-v0 task (vary the size of the finger). For environments with varying reward function, we use the Cheetah-Run-v1 environment (vary the target velocity that the agent has to reach) and Sawyer-Peg-v0 environment (vary the goal position for inserting the per) from~\citet{Zhao2020meld}. For environments with varying reward function, we assume access to the reward function, as done in~\citet{Zhao2020meld}.

For all environments, we pre-define a range of parameters to train and evaluate on. For environments with nonstationary transition dynamics, we create two set of parameters for evaluation \textit{interpolation} (and \textit{extrapolation}) where the parameters are sampled from a range that lies within (and outside) the range of parameters used for training. For the environments with varying reward function, we sample the parameters for the test environments from the same range as the training environments. For additional details refer to~\cref{appendix:setup}. We report the evaluation performance of the best performing hyper-parameters for all algorithms (measured in terms of the training performance). All the experiments are run with 10 seeds and we report both the mean and the standard error (denoted by the shaded area on the plots). For additional implementation details refer to~\cref{appendix:implementation_details}.

\subsection{Baselines}
\label{section:baselines}

We select representative baselines from different areas of related work (\cref{sec:related_work}):  \textit{UP-OSI}~\citep{yu2017osi} is a system identification approach that infers the true parameters (of the system) and conditioning the policy on the inferred parameters.~\textit{Context-aware Dynamics Model , CaDM}~\citep{lee2020cadm} is a context modelling based approach that is shown to outperform Gradient and Recurrence-based meta learning approaches~\citep{nagabandi2018learning}.~\textit{HyperDynamics}\citep{xian2021hyperdynamics} generates the weights of the dynamics model (for each environment) by conditioning on a context vector and is shown to outperform both ensemble of experts and meta-learning based approaches~\citep{nagabandi2018learning}. We also consider a~\textit{Context-conditioned Policy} where the context encoder is trained using the one-step forward dynamics loss. This approach can be seen as an ablation of the ZeUS algorithm without the context learning error (from~\cref{eq:total_loss}). We refer to it as \textit{Zeus-no-context-loss}.

\subsection{Adapting and generalizing to unseen environments}

\begin{figure}[t]
\centering
\subfigure[Cheetah-Run-v0]{\label{fig:cl-halfcheetah-v2-eval-extrapolation-all}\includegraphics[width=0.243\textwidth]{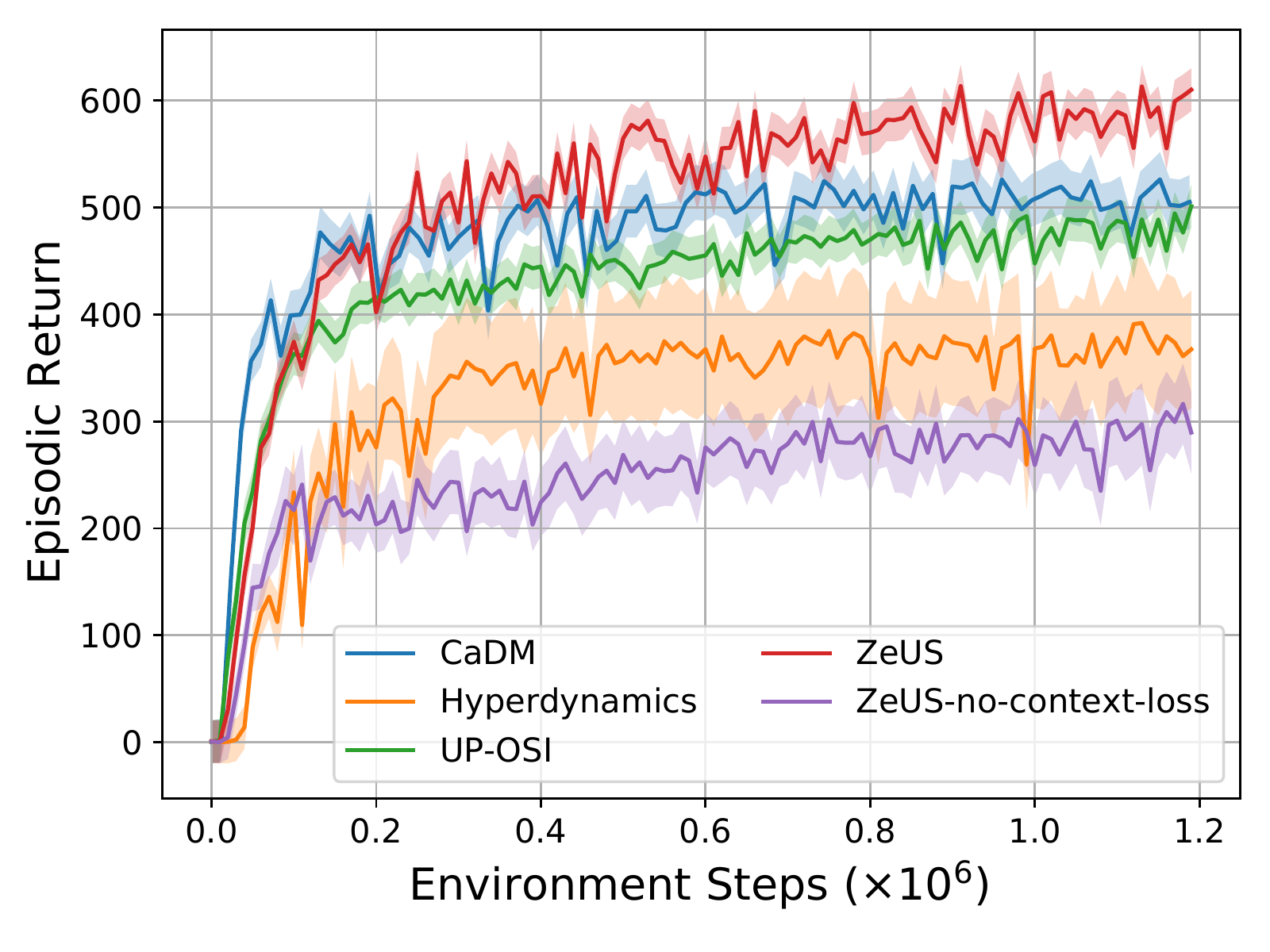}}
\subfigure[Finger-Spin-v0]{\label{fig:cl-finger-spin-v2-eval-extrapolation-all}\includegraphics[width=0.243\textwidth]{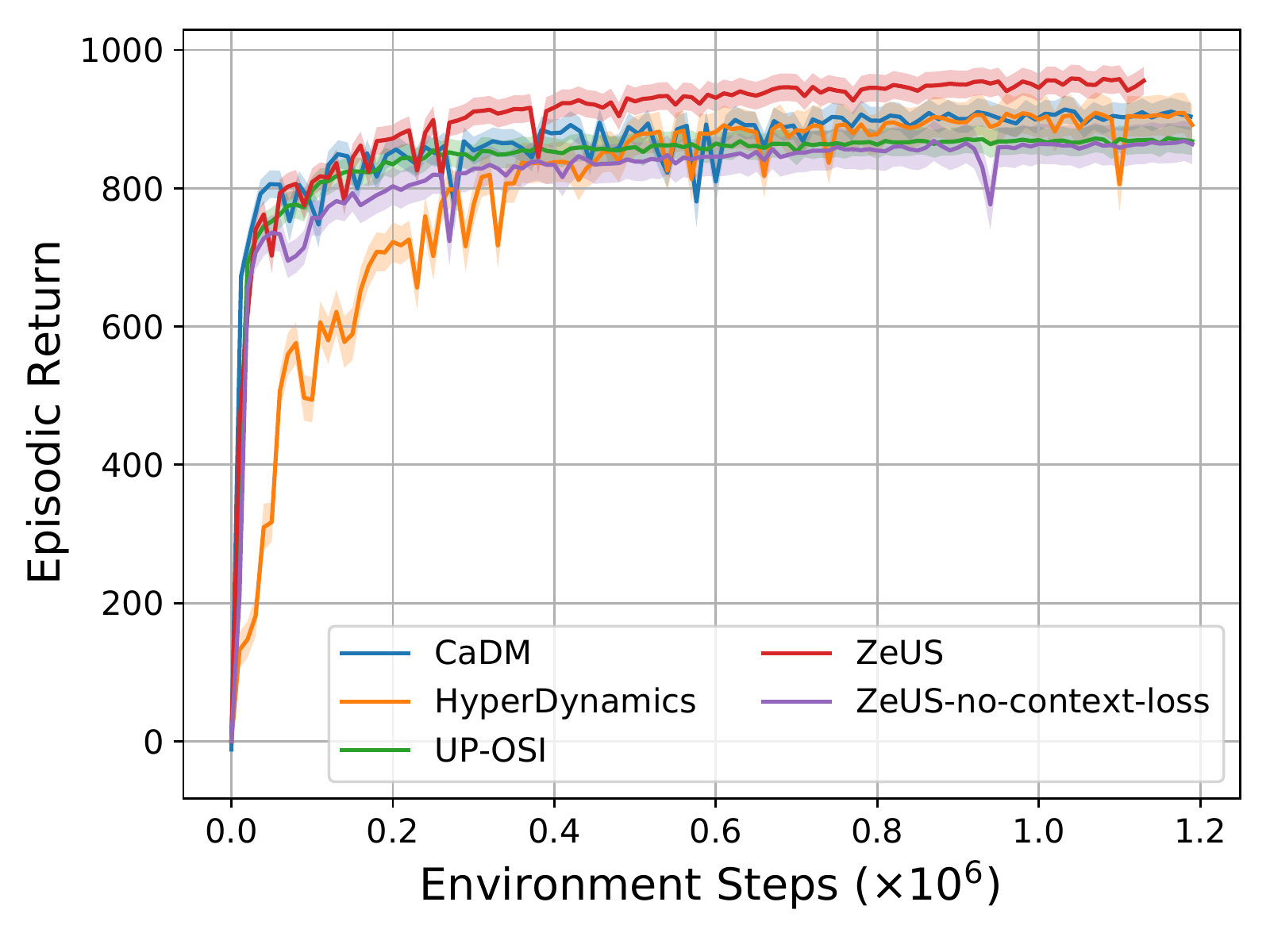}}
\subfigure[Walker-Walk-v0]{\label{fig:cl--walker-walk-v0-eval-extrapolation-all}\includegraphics[width=0.243\textwidth]{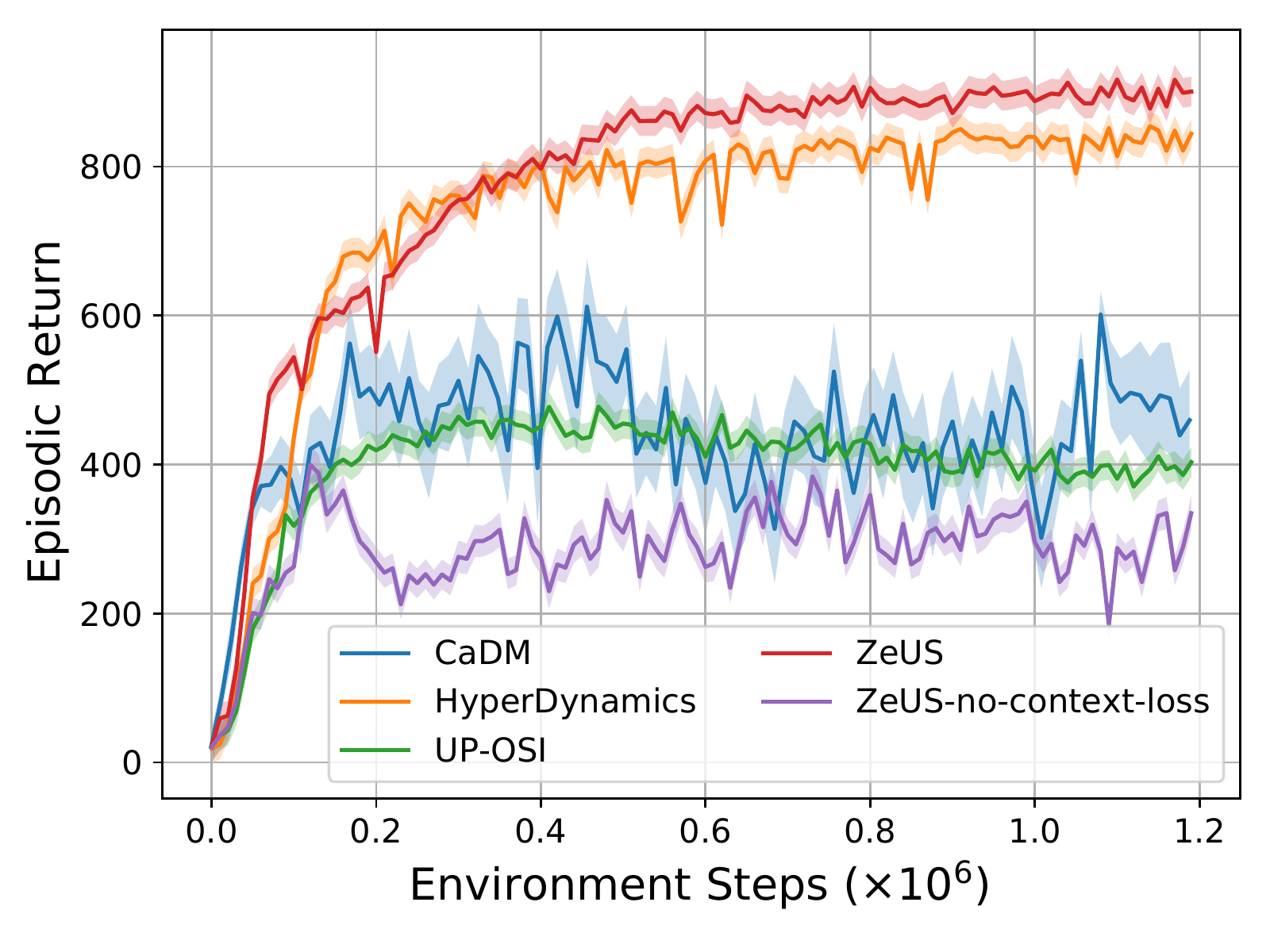}}
\subfigure[Walker-Walk-v1]{\label{fig:cl-walker-walk-v1-eval-extrapolation-all}\includegraphics[width=0.243\textwidth]{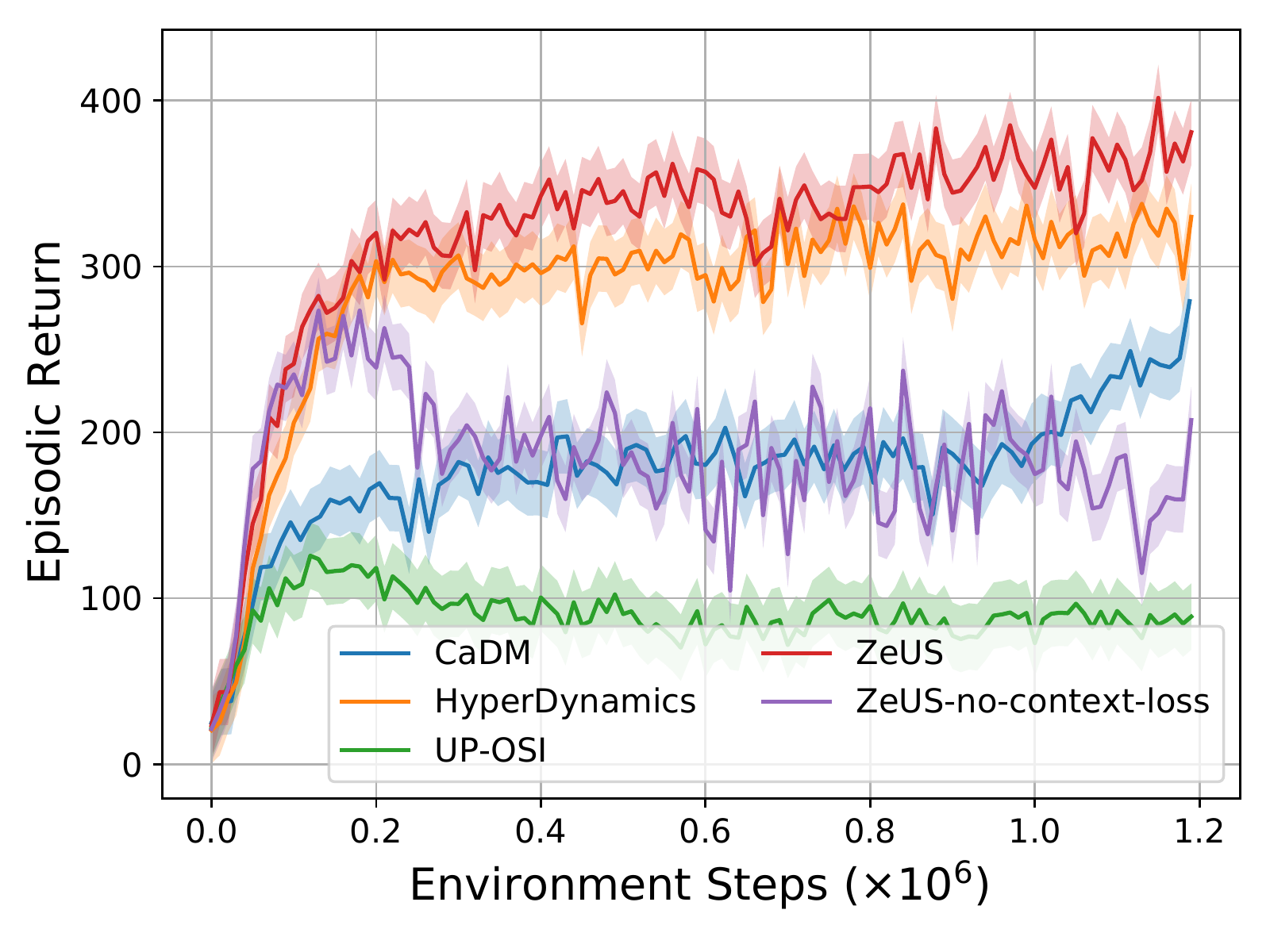}}
\caption{We compare the performance of the proposed ZeUS algorithm with~\textit{CaDM},~\textit{UP-OSI},~\textit{HyperDynamics} and~\textit{ZeUS-no-context-loss} algorithms on the heldout evaluation environments (extrapolation) for four families of tasks with different dynamics parameters.}  
\label{fig:cl-mujoco-all-eval_extrapolation}
\end{figure}

In~\cref{fig:cl-mujoco-all-eval_extrapolation}, we compare ZeUS's performance on the heldout \textit{extrapolation} evaluation environments which the agent has not seen during training. The  transition dynamics varies across these tasks.~\textit{HyperDynamics} performs well on some environments but requires more resources to train (given that it generates the weights of dynamics models for each transition in the training batch).~\textit{UP-OSI} uses privileged information (in terms of the extra supervision). Both~\textit{CaDM} and ZeUS are reasonably straightforward to implement (and train) though ZeUS outperforms the other baselines. The context loss (\cref{eq:total_loss}) is an important ingredient for the generalization performance as observed by the performance of \textit{Zeus-no-context-loss}. The corresponding plots for performance on the training environments and heldout \textit{interpolation} evaluation environments are given in~\cref{fig:cl-mujoco-all} and~\cref{fig:cl-mujoco-all-interpolation} (in Appendix) respectively. For additional ablation results for these environments, refer to~\ref{app:adapting_and_generalizing_to_unseen_environments}.

\begin{figure}[h!]
\centering
\subfigure[Cheetah-Run-v1]{\label{fig:cl-meld-cheetah-velocity-v2-all}\includegraphics[width=0.3\textwidth]{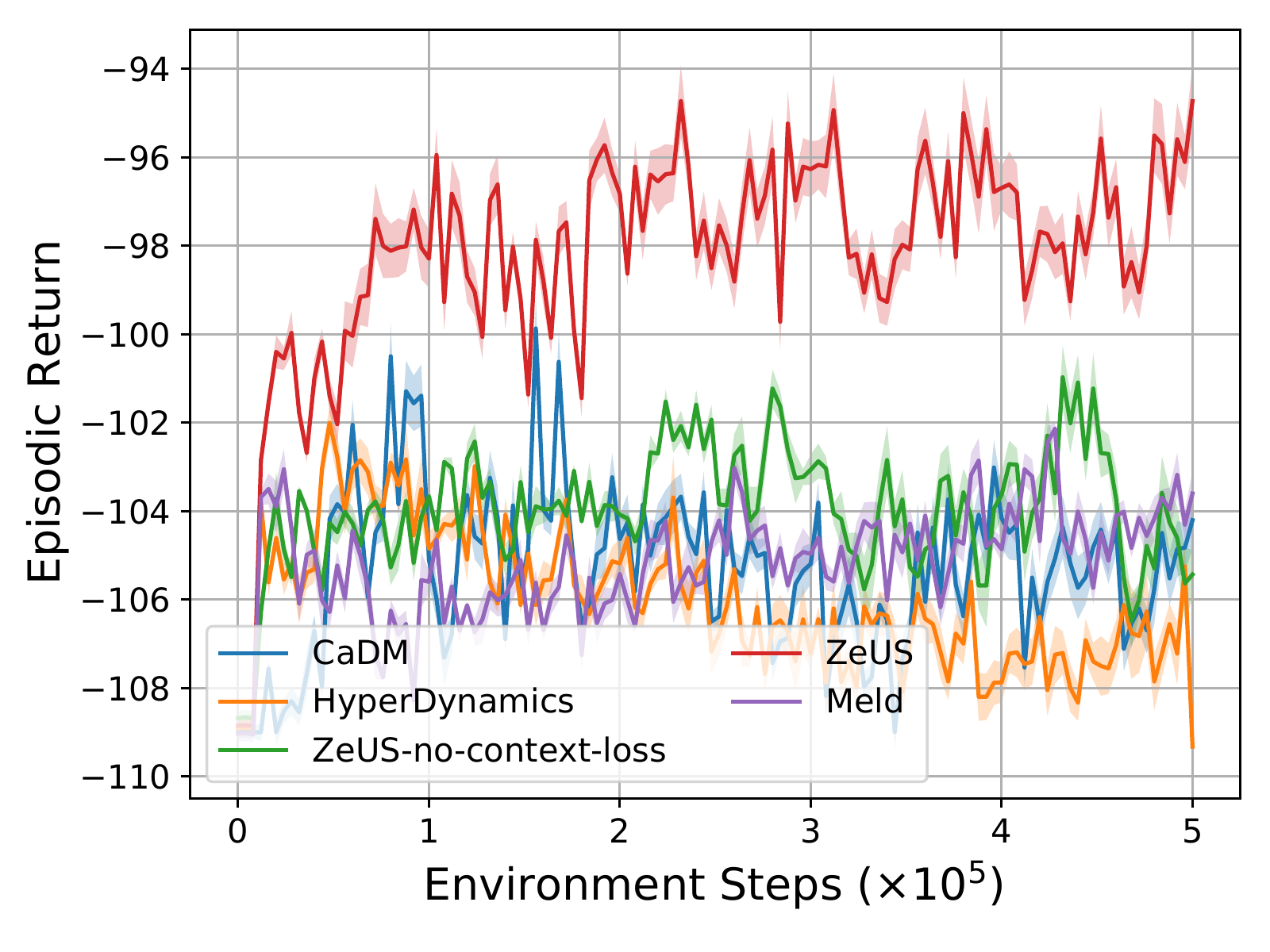}}
\subfigure[Sawyer-Peg-v0]{\label{fig:cl-meld-sawyer-peg-position-v1-all}\includegraphics[width=0.3\textwidth]{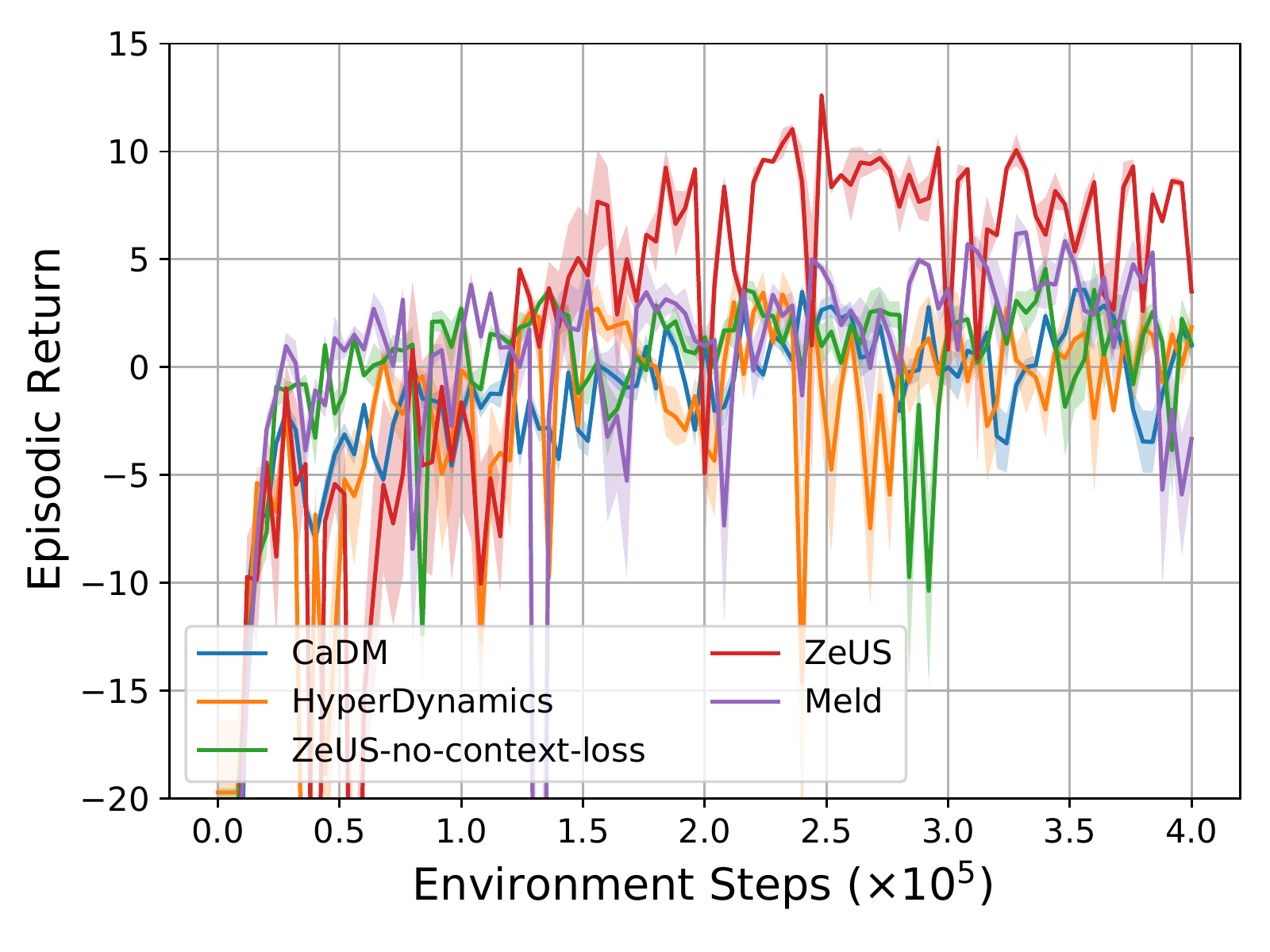}}
\hspace{30pt}
\subfigure[Sawyer-Peg-v0]{\label{fig:sawyer-peg-viz}\raisebox{5mm}{\includegraphics[width=0.19\textwidth]{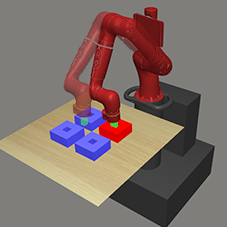}}}
\caption{(a), (b): We compare the performance of the proposed ZeUS algorithm with \textit{CaDM}, \textit{HyperDynamics}, \textit{ZeUS-no-bisim} and \textit{Meld} algorithms on environments with different reward functions. (c): Illustration of the Sawyer-Peg-V0 task.}
\label{fig:cl-meld-all-extrapolation}
\end{figure}

In~\cref{fig:cl-meld-all-extrapolation}, we compare ZeUS's performance with the baselines when the reward function varies across tasks. Since all the models have access to the reward (i.e. reward is concatenated as part of history), we do not compare with UP-OSI which is trained to infer the reward. Instead we include an additional baseline, Meld~\citep{zhu2020transfer}, a meta-RL approach that performs inference in a latent state model to adapt to a new task. Like before, ZeUS outperforms the other baselines. 

\subsection{Learning a Meaningful Context Representation}
\label{sec:meaningful}

\begin{wrapfigure}{r}{0.67\textwidth}
    \begin{center}
        \includegraphics[width=0.305\textwidth]{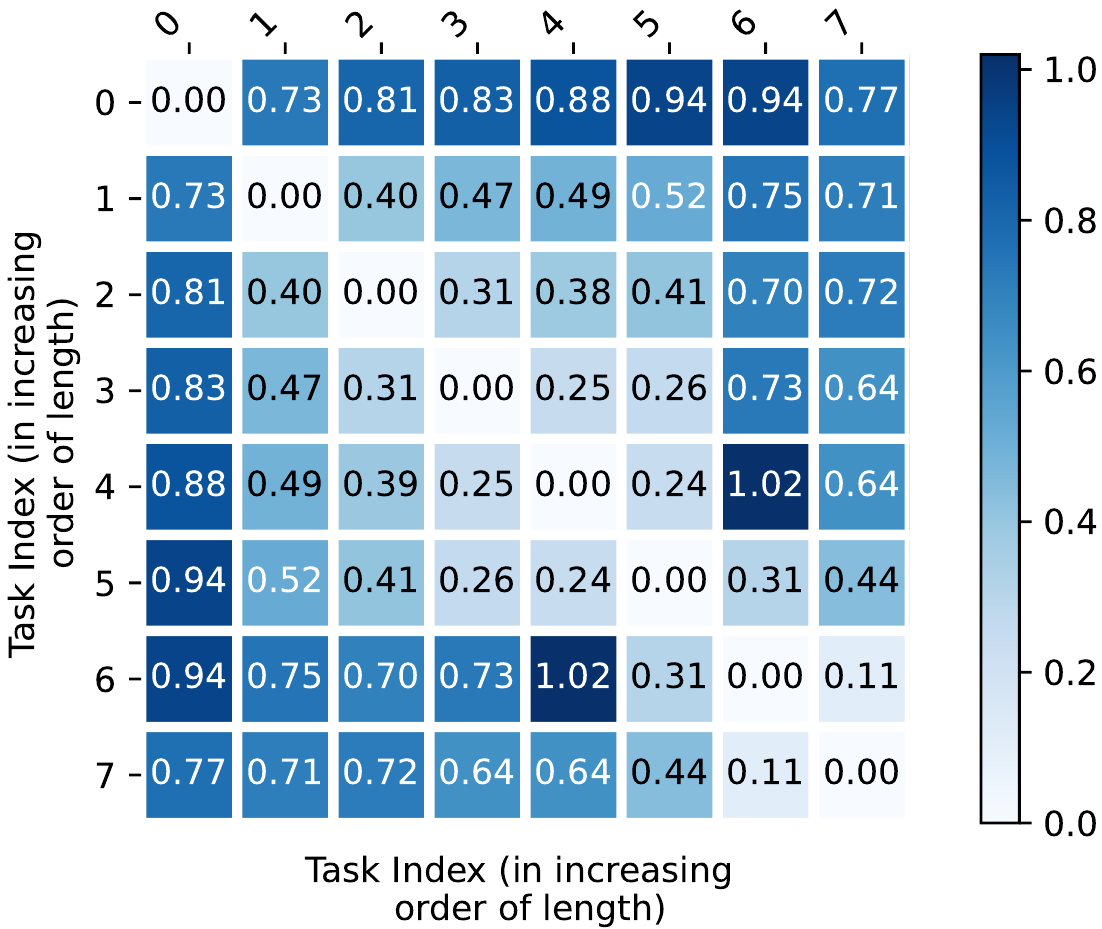}
        \includegraphics[width=0.325\textwidth]{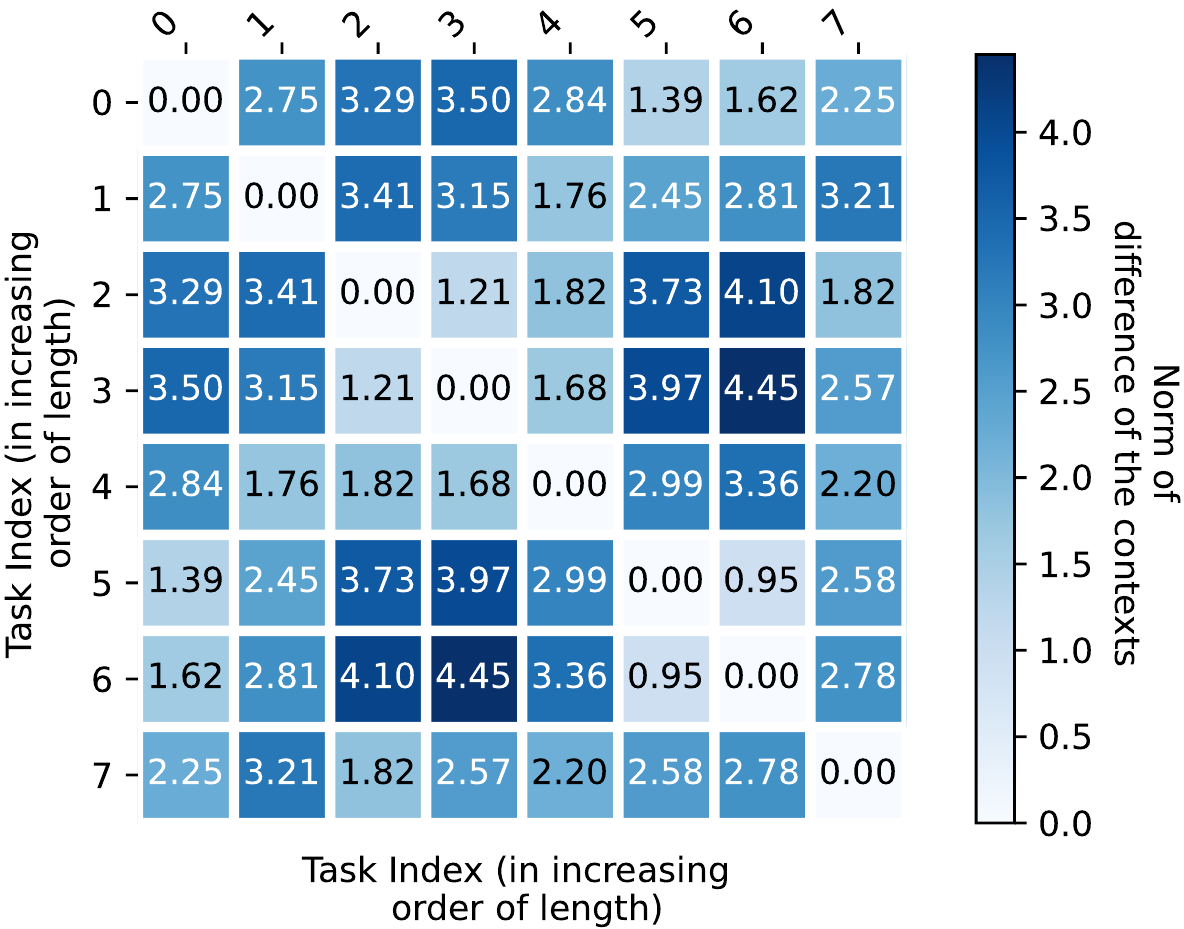}
    \end{center}
    \caption{Norm of pairwise differences of contexts for different tasks for Cheetah-Run-v0 setup when trained with context loss (left) and without context loss (right).
    }
\label{fig:task_distance_zeus}
\end{wrapfigure}

We want to evaluate if the context representation constructed by ZeUS contains meaningful information about the true context of the BC-MDP. We compute the norm of pairwise difference of the learned contexts (corresponding to different tasks) for the Cheetah-Run-v0 setup (where the torso of the cheetah varies across the tasks) when it is trained with and without the context loss (\cref{eq:total_loss}). As shown in~\cref{fig:task_distance_zeus} (left), when training with the context loss, tasks that are closer in terms of torso length are also \textit{consistently} closer in the context space and pairs of tasks with larger differences of torso length are also consistently further apart. We also compute the Spearman's rank correlation coefficient between the ranking (of distance) between the learned contexts and the ground truth context. The coefficient is much higher ($0.60$) when trained with the context loss than training without the context loss ($0.23$), showing that the context loss is useful for training representations that capture the relationship across tasks of the true underlying context variable without \textit{privileged information} of task ids or the true context.

\section{Limitations}
\label{sec:limitations}

\begin{wrapfigure}{r}{0.45\textwidth}
\vspace{-40pt}
    \begin{center}
        \includegraphics[width=0.45\textwidth]{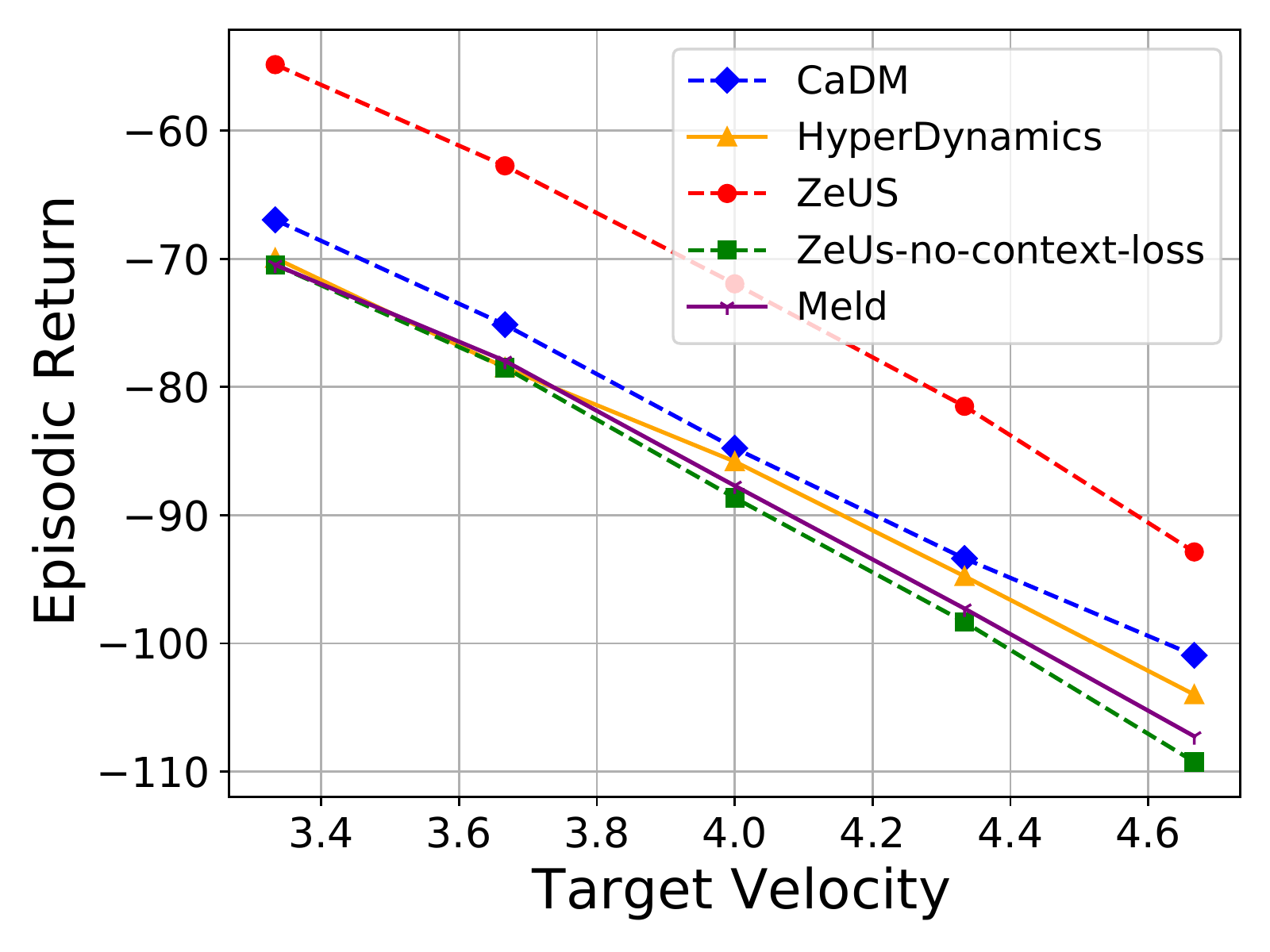}
    \end{center}
    \caption{The performance of all the algorithms (on Cheetah-Run-v1) degrades as we move away from the training distribution.}
\label{fig:degrading_performance}
\vspace{-10pt}
\end{wrapfigure}

A theoretical limitation of this work is the inability to provide guarantees based on the likelihood of the model learning the correct causal dynamics. By structuring the context space to be Lipschitz, we can give guarantees only for those dynamics and reward where the context is close to the contexts seen at training time.  While this result flows directly from~\cref{thm:gen_bound}, it is important to be aware of this limitation when using ZeUS in practice, namely that it may have poor performance when the distance between the contexts (corresponding to the training and the evaluation tasks) is high. We demonstrate an example in~\cref{fig:degrading_performance}, where we plot the performance of the agent for different values of target velocities (for Cheetah-Run-v1). While ZeUS outperforms the other methods, its performance also degrades as we move away from the training distribution.

\section{Discussion}
\label{sec:Discussion}
In this work, we propose to use the Block Contextual MDP framework~\citep{sodhani2021multi} to model the nonstationary, rich observation, RL setting. This allows classical RL methods that typically rely on a stationarity assumption to be used in continual learning, and we show that our adaptive approach can zero-shot generalize to new, unseen environments.  We provide theoretical bounds on adaptation and generalization ability to unseen tasks within this framework and propose a representation learning algorithm (ZeUS) for performing online inference of ``unknown unknowns''. We empirically verify the effectiveness of the proposed algorithm on environments with nonstationary dynamics and reward functions and show that a crucial component of ZeUS is the context loss that ensures smoothness in the context space. This context loss successfully captures meaningful information about the true context, as verified in \cref{sec:meaningful}. 

There are a few interesting directions for further research. One way to improve the generalization performance and tighten the generalization bounds is by constraining the neural networks used in ZeUS to have smaller Lipschitz constants. This is known to be able to improve generalization bounds~\citep{pmlr-v40-Neyshabur15}. We can also consider improving the algorithm to infer the underlying causal structure of the dynamics, as discussed in \cref{sec:limitations}. This is a much harder problem than constructing a context space and inferring context in new environments. A second direction is to extend ZeUS to account for \textit{active nonstationarity}, where the agent's actions can affect the environment. ZeUS would work for this setting, but there is clearly additional structure that can be leveraged for improved performance. There are connections to existing work in multi-agent RL.

\section{Reproducibility Statement}

We provide implementation related details in~\cref{appendix:implementation_details}. The pseudocode for the main algorithm and the the context loss update are provided in~\cref{alg:zeus} and~\cref{alg:update_using_context_loss} respectively.~\cref{appendix:setup} contains details related to the compute resources/time,~\cref{sec:env_details} contains the information related to the different tasks and~\cref{app:license} contains the information related to software stack and licenses. Hyper-parameters, for the different experiments, are enlisted in~\cref{app:hyperparameters}. All the experiments are run with 10 seeds and we report both the mean and the standard error (denoted by the shaded area on the plots)~\cref{sec:setup}. We will be open-sourcing the code and including a link for the same in the camera-ready version.

\section{Acknowledgment}
We thank Alessandro Lazaric, Kamalika Chaudhuri and Nicolas Usunier for detailed feedback that improved this manuscript.

\bibliography{iclr2022_conference}
\bibliographystyle{iclr2022_conference}

\clearpage

\appendix

\section{Additional Related Work}
\label{app:related_work}

This section extends the discussion on the related works from Section~\ref{sec:related_work}.

Some prior works train a~\textbf{single fixed dynamics model} but introduce additional constraints that ensure the latent dynamics are locally linear~\citep{watter2015embed, banijamali2018robust} or that the learned dynamics model is invariant to translation dynamics~\citep{fragkiadaki2015learning}. However, such approaches fail when the parameters of the underlying dynamics model change. In theory, one could learn a new~\textit{expert} dynamics model per task, but that is computationally expensive and requires access to the unseen environments.~\citet{xian2021hyperdynamics} generates the weights of an \textit{expert} dynamics model by using the environment context as an input. Specifically, they proposed an algorithm called HyperDynamics where they use a HyperNetwork~\citep{ha2016hypernetworks, chang2019principled, klocek2019hypernetwork, meyerson2019modular} to generate the weights of the dynamics model.Similar to our work,~\citet{lee2020cadm} also introduced additional loss terms when training the agent. However, their objective is to encourage the context encoding to be useful for predicting both forward (next state) and backward (previous state) dynamics while being temporally consistent.~\cite{asadi18lipschitzmbrl} present results when bounding the model class to Lipschitz functions but assume that the given MDP is Lipschitz. We instead show that this constraint can be placed on the representation learning objective for better results in the continual learning setting, even when the original BC-MDP is not Lipschitz. 

Our work is related to the broad area of \textbf{multitask RL and transfer learning}~\citep{caruana1997multitask_learning, zhang2014facial_landmark_detection_by_deep_multitask_learning, kokkinos2017ubernet, radford2019language_models_are_unsupervised_multitask_learners, epopt_learning_robust_neural_network_policies_using_model_ensembles}. Previous works have looked at multi-task and transfer learning from the perspective of policy transfer~\citep{rusu2015policy, yin2017knowledge}, policy reuse~\citep{fernandez2006probabilistic, barreto2016successor}, representation transfer~\citep{rusu2016progressive, devin2017learning, andreas2017modular, sodhani2021multi} etc. One unifying theme in these works is that the agent is finetuned on the new environment while we focus on setup where there are no gradient updates when evaluating on the unseen environment.~\citet{zhang2021hipbmdp} also uses task metrics to learn a context space, but focus on the rich observation version of the  Hidden-Parameter MDP (HiP-MDP) setting~\citep{doshi-velezHiddenParameterMarkov2013} with access to task ids. Similarly,~\citet{perezGeneralizedHiddenParameter2020} also treats the multi-task setting as a HiP-MDP by explicitly designing latent variable models to model the latent parameters, but require knowledge of the structure upfront, whereas our approach does not make any such assumptions. 

Several works proposed a~\textit{mixture-of-experts} based approach to adaptation where an expert model is learned for each task~\citep{doya2002multiple, neumann2009learning, peng2016terrain} or experts are shared across tasks~\citep{multitask_soft_option_learning, sodhani2021multi}.~\citet{distral} additionally proposed to distill these experts in a single model that should work well across multiple tasks. While these systems perform reasonably well on the training tasks, they do not generalize well to unseen tasks.

\section{Implementation Details}
\label{appendix:implementation_details}
We provide additional details of the main algorithm, environment setup, compute used, baselines, and hyperparameter information.

\subsection{Algorithm}
We provide pseudocode for the main algorithm in \cref{alg:zeus} and the context loss update in \cref{alg:update_using_context_loss}.
\begin{algorithm}[t]
\begin{algorithmic}[1]
\REQUIRE Actor, Critic, Dynamics Model $T$, Reward Model $R$, Observation Encoder $\phi$, Context Encoder $\psi$, Replay Buffer $\mathcal{D}$.
\FOR{each timestep $t=1..T$}
    \FOR{each $\task_i$}
        \STATE  $a^i_t \sim \pi^i(\cdot | s^i_{t})$
        \STATE  $s'^i_t \sim p^i(\cdot | s^i_{t}, a^i_t)$
        \STATE $\mathcal{D} \leftarrow \mathcal{D} \cup (s^i_t, a^i_t, r(s^i_t, a^i_t), s'^i_t)$
        \STATE \textsc{Update\_Critic}($\mathcal{D}$)
        \STATE \textsc{Update\_Actor}($\mathcal{D}$)
        \STATE \textsc{Update\_Using\_Context\_Loss}($\mathcal{D}$)
    \ENDFOR
\ENDFOR
\end{algorithmic}
\caption{{\bf Training ZeUS algorithm.}}
\label{alg:zeus}
\label{alg:algorithm}
\end{algorithm}

\begin{algorithm}[H]
\begin{algorithmic}[1]
\REQUIRE Batches of data sampled from the Replay Buffer $\mathcal{D}$,
 Dynamics Model $T$, Reward Model $R$, Observation Encoder $\phi$, Context Encoder $\psi$, Learning rates $\alpha_\psi$, $\alpha_T$ and $\alpha_R$.
\FOR{each batch}
    \STATE {Compute $\mathcal{L}(\phi,\psi,T,R))$ using  \cref{eq:total_loss}}
\STATE $\psi \gets \psi - \alpha_\psi  \nabla_\theta \sum_i \mathcal{L}$
\STATE $T \gets T - \alpha_T  \nabla_\theta \sum_i \mathcal{L}$
\STATE $R \gets R - \alpha_R  \nabla_\theta \sum_i \mathcal{L}$
\ENDFOR
\end{algorithmic}
\caption{{\bf Update Using Context Loss}}
\label{alg:update_using_context_loss}
\end{algorithm}

\subsection{Setup}
\label{appendix:setup}

\paragraph{Environments}

For tasks with varying dynamics, we use the Finger, Cheetah and Walker environments. These are Mujoco~\citep{todorov2012mujoco}\footnote{\href{https://www.roboti.us}{https://www.roboti.us}} based tasks that are available as part of the DMControl Suite~\cite{deepmindcontrolsuite2018}\footnote{https://github.com/deepmind/dm\_control}. We use MTEnv~\cite{Sodhani2021MTEnv}\footnote{https://github.com/facebookresearch/mtenv/commit/7fdec15f7e842bce4c17f4f3328d9d6fdc79d7fc} to interface with the environments. In the Finger Spin task, the agent has to manipulate (continually rotate) a planar \textit{finger}. We vary the size across different tasks and the specific values (for each mode) are listed in~\cref{table::finger_spin_params}. In the Cheetah Run task, a planar bipedal cheetah has to run as fast as possible (up to a maximum velocity). In the Walker Walk tasks, a planar walker has to learn to walk. We consider two cases here: (i) varying the friction coefficient between the walker and the ground, and (ii) the length of the foot of the walker.

For tasks with varying reward function, we use the Cheetah and the Sawyer Peg environments from~\cite{Zhao2020meld}\footnote{\href{https://github.com/tonyzhaozh/meld}{https://github.com/tonyzhaozh/meld}}. We highlight some challenges in the evaluation: In the cheetah environment, the reward depends on the difference in the magnitude of the agent's velocity and the target velocity. In each run of the algorithm, the target velocities are randomly sampled. We noticed that the cheetah can easily match the target velocity when the target velocity is small and makes a larger error when the target velocity is large. In practice, the performance of the cheetah (as measured in terms of episodic rewards) can vary a lot depending on the sampled target velocities. To make the comparison fair across the different baselines, we fix the values of the target velocities ( for train and for eval) instead of sampling them. In the Peg-Insertion task, the reward has an extra ``offset'' term as shown in the~\href{https://github.com/tonyzhaozh/meld/blob/271cd6d6a002a6c1a637a908a6e3c06c8349fcbd/meld/environments/envs/reacher/sawyer_reacher.py#L179}{\textcolor{blue}{implementation}} but this does not match the description in ~\cite{Zhao2020meld}. We use the version of reward function without the offset, as described in~\cite{Zhao2020meld}.

All the algorithms are implemented using PyTorch~\citep{paszke2017automatic}\footnote{\href{https://pytorch.org/}{https://pytorch.org/}}. We use the MTRL codebase\footnote{https://github.com/facebookresearch/mtrl/commit/eea3c99cc116e0fadc41815d0e7823349fcc0bf4} as the starting codebase. 

\paragraph{Compute Resources}
Unless specified otherwise, all the experiments are run using Volta GPUs with 16 GB memory. Each experiment uses 1 GPU and 10 CPUs (to parallelize the execution in the environments). The experiments with HyperDynamics model are run with Volta GPUs with 32 GB Memory. For the Cheetah-Run-v0, Finger-Spin-v0, Walker-Walk-v0 and Walker-Walk-v1, training CADM and ZeUS takes about 44 hours, training HyperDynamics takes about 62 hours and training UP-OSI takes 24 hours (all for 1.2 M steps). For Cheetah-Run-v1 and Sawyer-Peg-v1, training CADM and ZeUS takes about 25 hours while HyperDynamics takes about 32.5 hours (all for 300K steps). These times include the time for training and evaluation.

In~\cref{table::assumption_identifiability}, we show that a two-layer feedforward network can be trained to infer the context value from last $5$ observation-action tuples. The setup is modeled as a regression problem where the observation-action transition tuples are obtained from a random policy.

\begin{table}[H]
\centering
\caption{Loss when training the model to infer the context}
\label{table::assumption_identifiability}
\resizebox{0.5\textwidth}{!}
{
    \begin{tabular}{p{3cm}p{3cm}}
    \toprule
    Environment Name      & Loss \\
    \midrule
        Cheetah-Run-v0            & $5.12 \times 10^{-5}$ \\
        Finger-Spin-v0            & $7.45 \times 10^{-5}$ \\
        Walker-Walk-v0            & $7.22 \times 10^{-5}$ \\
        Walker-Walk-v1            & $6.41 \times 10^{-5}$ \\
    \bottomrule
    \end{tabular}
}
\end{table}

\subsubsection{Baselines}
\label{appendix:baselines}

We provide additional implementation related details for the baselines. For a summary of the baselines, refer Section~\ref{section:baselines}.

\begin{enumerate}
    \item \textit{Context-aware Dynamics Model (CaDM)}:~\citet{lee2020cadm} proposed a two stage pipeline: (a) use a context encoder to obtain a context vector given the recent interactions and (ii) perform online adaption by conditioning the forward dynamics model on the context vector. CaDM is shown to outperform vanilla dynamics models which do not use the interaction history, stacked dynamics models which use the interaction history as an input, and Gradient and Recurrence-based meta learning approaches~\citep{nagabandi2018learning}.
    
    \item \textit{UP-OSI}:~\citet{yu2017osi} proposed learning two components: (i) Universal Policy (UP) and (ii) On-line System Identification model (OSI). The universal policy is trained over a wide range of dynamics parameters so that it can operate in a previously unseen environment (given access to the parameters of the dynamics model). It is modeled as a function of the dynamic parameters $\mu$, such that $a_{t}=\pi_{\mu}(s_{t})$ and is trained offline, in simulation, without requiring access to the real-world samples. The goal of the universal policy is to generalize to the space of the dynamics models. The on-line system identification model is trained to identify the parameters of the dynamics model conditioned on the last $k$ state-action transition tuples i.e. $\mu = \phi((s_i,a_i,s_i',r_i), \forall i\in\{1,...,k\})$, where $\phi$ is the OSI module. $\phi$ is trained using a supervised learning loss by assuming access to true parameters of the dynamics model. During evaluation (in an unseen environment), the OSI module predicts the dynamic parameters at every timestep. The universal policy uses these inferred dynamics parameters to predicts the next action.The system identification module is trained to identify a model that is good enough to predict a policy’s value function that is similar to the current optimal policy. Following the suggestion by~\citep{yu2017osi}, we initially train the OSI using the data collected by UP (when using true model parameters) and then switch to the data collected using inferred parameters. Specifically, the paper suggested using first 3-5 iterations (out of 500 iterations) for training with the ground truth parameters. We scaled the number of these warmup iterations to match the number of updates in our implementation and report the results using the same. During evaluation, we report the performance of UP using the inferred parameters (UP-\textit{inferred}), as recommended in the paper.

    \item \textit{HyperDynamics}:~\citet{xian2021hyperdynamics} proposed learning three components: (i) an encoding module that encodes the agent-environment interactions, (ii) a hypernetwork that generates the weight for a dynamics model at the current timestep, and (iii)  a (target) dynamics model that uses the weights generated by the hypernetwork. All the components (and the policy) are trained jointly in an end-to-end manner. HyperDynamics is shown to ourperform both ensemble of experts and meta-learning based approaches~\citep{nagabandi2018learning}.
    
    \item \textit{Context-conditioned Policy}: We train a context encoder that encoders the interaction history into latent context that is passed to the policy and the dynamics. This approach can be thought of as an ablation of the ZeUS algorithm with $\alpha_\psi = 0$, or without the context learning error (from~\cref{eq:total_loss}). We label this baseline as \textit{Zeus-no-context-loss}.
\end{enumerate}

We note that when training on the environments with varying reward dynamics, we concatenate the reward along with the environment observation (as done in~\cite{Zhao2020meld}). 

\subsection{Environment Details}
\label{sec:env_details}
\begin{wrapfigure}{r}{0.29\linewidth}
    \centering
    \includegraphics[width=0.3\linewidth]{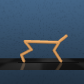}
    \includegraphics[width=0.3\linewidth]{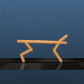}
    \includegraphics[width=0.3\linewidth]{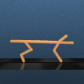}
    \caption{\small Variation in \texttt{Cheetah-Run-v0} tasks.}
    \label{fig::env::cheetah-v0}
\end{wrapfigure}

\begin{table}[H]
\centering
\caption{Parameter values for different modes for the Finger environments (Finger-Spin-v0) when varying the size of the finger across the tasks.}
\label{table::finger_spin_params}
\resizebox{1.0\textwidth}{!}
{
    \begin{tabular}{p{3.7cm}p{10.7cm}}
    \toprule
    Mode      & Values \\
    \midrule
        Train / Eval       & {[}0.05, 0.0625, 0.075, 0.0875, 0.15, 0.1625, 0.175, 0.1875{]} \\
        Eval Interpolation & {[}0.1,  0.1125,  0.125, 0.1375{]}                             \\
        Eval Extrapolation & {[}0.01, 0.0125, 0.025, 0.0375, 0.2, 0.2125, 0.35, 0.5{]}      \\
    \bottomrule
    \end{tabular}
}
\end{table}

\begin{table}[H]
\centering
\caption{Parameter values for different modes for the Cheetah environments (Cheetah-Run-v0) when varying the length of cheetah's torso across the tasks.}
\label{table::cheetah_torso_params}
\resizebox{1.0\textwidth}{!}
{
    \begin{tabular}{p{3.7cm}p{10.7cm}}
    \toprule
    Mode      & Values \\
    \midrule
        Train / Eval       & {[}0.6, 0.65, 0.7, 0.75, 1, 1.05, 1.1, 1.15, 1.25, 1.4, 1.5{]} \\
        Eval Interpolation & {[}0.8,  0.85,  0.9, 0.95{]}                             \\
        Eval Extrapolation & {[}0.3, 0.4, 0.5, 0.55, 1.2, 1.25, 1.4, 1.5{]}      \\
    \bottomrule
    \end{tabular}
}
\end{table}

\begin{table}[H]
\centering
\caption{Parameter values for different modes for the Walker environments (Walker-Walk-v0) when varying the friction coefficient between the walker and the ground.}
\label{table::walker_friction_params}
\resizebox{1.0\textwidth}{!}
{
    \begin{tabular}{p{3.7cm}p{10.7cm}}
    \toprule
    Mode      & Values \\
    \midrule
        Train / Eval       & {[}0.8, 0.85, 0.9, 0.95, 1.2, 1.25, 1.3, 1.35{]} \\
        Eval Interpolation & {[}1.0, 1.05, 1.1, 1.15{]}                             \\
        Eval Extrapolation & {[}0.3, 0.5, 0.7, 0.75, 1.4, 1.45, 1.7, 1.9{]}      \\
    \bottomrule
    \end{tabular}
}
\end{table}

\begin{table}[H]
\centering
\caption{Parameter values for different modes for the Walker environments (Walker-Walk-v1) when varying the length of walker's foot.}
\label{table::walker_foot_params}
\resizebox{1.0\textwidth}{!}
{
    \begin{tabular}{p{3.7cm}p{10.7cm}}
    \toprule
    Mode      & Values \\
    \midrule
        Train / Eval       & {[}0.09, 0.095, 0.1, 0.105, 0.13, 0.135, 0.14, 0.145{]} \\
        Eval Interpolation & {[}0.11, 0.115, 0.12, 0.125{]}                             \\
        Eval Extrapolation & {[}0.03, 0.06, 0.08, 0.085, 0.15, 0.155, 0.18, 0.21{]}      \\
    \bottomrule
    \end{tabular}
}
\end{table}

\begin{table}[H]
\centering
\caption{Values for the target velocity for the Cheetah environments (Cheetah-Run-v1).}
\label{table::cheetah_run_params}
\resizebox{1.0\textwidth}{!}
{
    \begin{tabular}{p{3.7cm}p{10.7cm}}
    \toprule
    Mode      & Values \\
    \midrule
        Train & {[}0., 0.43, 0.86, 1.29, 1.71, 2.14, 2.57, 3.0{]} \\
        Eval  & {[}0.4, 0.78, 1.11, 1.47, 1.83, 2.18, 2.55, 2.9{]} \\
    \bottomrule
    \end{tabular}
}
\end{table}

\begin{table}[H]
\centering
\caption{Range for sampling the values for the Sawyer-Peg environments (Sawyer-Peg-v0)}
\label{table::sawyer_peg_params}
\resizebox{1.0\textwidth}{!}
{
    \begin{tabular}{p{3.7cm}p{10.7cm}}
    \toprule
    Mode      & Values \\
    \midrule
        x\_range\_1      & {(}0.44, 0.45{)} \\
        x\_range\_2      & {(}0.6, 0.61{)} \\
        y\_range\_1      & {(}-0.08, -0.07{)} \\
        y\_range\_2      & {(}0.07, 0.08{)} \\
    \bottomrule
    \end{tabular}
}
\end{table}
                
\subsection{License}
\label{app:license}

\begin{enumerate}
    \item Mujoco: Commercial (with Trial Version) \href{https://www.roboti.us/license.html}{https://www.roboti.us/license.html}
    \item DeepMind Suite: Apache \href{https://github.com/deepmind/dm\_control/blob/master/LICENSE}{https://github.com/deepmind/dm\_control/blob/master/LICENSE}
    \item MTEnv: MIT License \href{https://github.com/facebookresearch/mtenv/blob/main/LICENSE}{https://github.com/facebookresearch/mtenv/blob/main/LICENSE}
    \item Meld Codebase: \href{https://github.com/tonyzhaozh/meld}{https://github.com/tonyzhaozh/meld}
    \item PyTorch: \href{https://github.com/pytorch/pytorch/blob/master/LICENSE}{https://github.com/pytorch/pytorch/blob/master/LICENSE}
    \item MTRL: MIT License \href{https://github.com/facebookresearch/mtrl/blob/main/LICENSE}{https://github.com/facebookresearch/mtrl/blob/main/LICENSE}
    \item Hydra: MIT License \href{https://github.com/facebookresearch/hydra/blob/master/LICENSE}{https://github.com/facebookresearch/hydra/blob/master/LICENSE}
\end{enumerate}

\subsection{Hyperparameter Details}
\label{app:hyperparameters}

In this section, we provide hyper-parameter values for each of the methods in our experimental evaluation. We also describe the search space for each hyperparameter. In~\cref{table::common_hp_rewards} and~\cref{table::common_hp_dynamics} , we provide the hyperparameter values that are common across all the methods.

\begin{table}[h]
\centering
\caption{Hyperparameter values that are common across all the methods for Cheetah-Run-v0, Finger-Spin-v0, Walker-Walk-v0 and Walker-Walk-v1 (envs with varying dynamics)}
\label{table::common_hp_dynamics}
\resizebox{0.7\textwidth}{!}
{
    \begin{tabular}{p{3.7cm}p{6.7cm}}
    \toprule
         Hyperparameter & Hyperparameter values \\
    \midrule
         batch size (per task) & 128 \\
         network architecture & feedforward network\\
         non-linearity & ReLU \\
         policy initialization & standard Gaussian\\
         exploration parameters & run a uniform exploration policy 1500 steps\\
         \# of samples / \# of train steps per iteration & 1 env step / 1 training step\\
         policy learning rate & 3e-4 \\
         Q function learning rate & 3e-4 \\
         Critic update frequency & 2 \\
         optimizer & Adam\\
         beta for Adam optimizer for policy & (0.9, 0.999) \\
         Q function learning rate & 3e-4 \\
         beta for Adam optimizer for Q function & (0.9, 0.999) \\
         Discount & .99 \\
         Episode length (horizon) & 500 \\
         Reward scale & 1.0 \\
         actor update frequency & 2 \\
         actor log stddev bounds & $[-10, 2]$ \\
         number of layers in actor/critic & 2 \\
         actor/critic hidden dimension & 1024 \\
         number layers in dynamics model & 1 \\
         dynamics hidden dimension & 512 \\
         number of encoder layers & $4$ \\
         number of filters in encoder & $32$ \\
         Replay buffer capacity & 400000 \\
         Temperature Adam's $\beta_1$ & $0.9$ \\
         Init temperature & 0.1 \\
         Context Length & 5 \\
    \bottomrule
    \end{tabular}
}
\end{table}

\begin{table}[h]
\centering
\caption{Hyperparameter values that are common across all the methods for Cheetah-Run-v1 and Sawyer-Peg-v0 (envs with varying reward functions)}
\label{table::common_hp_rewards}
\resizebox{0.7\textwidth}{!}
{
    \begin{tabular}{p{3.7cm}p{6.7cm}}
    \toprule
         Hyperparameter & Hyperparameter values \\
    \midrule
         batch size (per task) & 128 \\
         network architecture & feedforward network\\
         non-linearity & ReLU \\
         policy initialization & standard Gaussian\\
         exploration parameters & run a uniform exploration policy 10000 steps \\
         \# of samples / \# of train steps per iteration & 1 env step / 1 training step\\
         policy learning rate & 3e-4 \\
         Q function learning rate & 3e-4 \\
         Critic update frequency & 2 \\
         optimizer & Adam\\
         beta for Adam optimizer for policy & (0.5, 0.999) \\
         Q function learning rate & 3e-4 \\
         beta for Adam optimizer for Q function & (0.5, 0.999) \\
         Discount & .99 \\
         Episode length (horizon) & 40 \\
         Reward scale & 1.0 \\
         actor update frequency & 2 \\
         actor log stddev bounds & $[-10, 2]$ \\
         number of layers in actor/critic & 2 \\
         actor/critic hidden dimension & 1024 \\
         number layers in dynamics model & 1 \\
         dynamics hidden dimension & 512 \\
         number of encoder layers & $4$ \\
         number of filters in encoder & $32$ \\
         Replay buffer capacity & 400000 \\
         Temperature Adam's $\beta_1$ & $0.5$ \\
         Init temperature & 0.1 \\
         Context Length & 5 \\
    \bottomrule
    \end{tabular}
}
\end{table}

\begin{table}[H]
\centering
\caption{Hyperparameter values for Context Aware Dynamics Model}
\label{table::multitask_cad_hp}
\resizebox{0.7\textwidth}{!}
{
    \begin{tabular}{p{3.7cm}p{3.7cm}p{3.7cm}}
    \toprule
         Hyperparameter & Hyperparameter values & Environment \\
    \midrule
         $\beta$ & 0.5 & Cheetah-Run-v0  \\
         $\beta$ & 0.5 & Finger-Spin-v0  \\
         $\beta$ & 0.5 & Walker-Walk-v0  \\
         $\beta$ & 0.5 & Walker-Walk-v1  \\
         $\beta$ & 0.5 & Cheetah-Run-v1  \\
         $\beta$ & 0.5 & Sawyer-Peg-v0  \\
    \bottomrule
    \end{tabular}
}
\end{table}

\begin{table}[h]
\centering
\caption{Hyperparameter values for ZeUS}
\label{table::multitask_zeus_hp}
\resizebox{0.7\textwidth}{!}
{
    \begin{tabular}{p{3.7cm}p{3.7cm}p{3.7cm}}
    \toprule
         Hyperparameter & Hyperparameter values & Environment \\
    \midrule
         $\alpha_\psi$ & 1.0 & Cheetah-Run-v0  \\
         $\alpha_\psi$ & 0.5 & Finger-Spin-v0  \\
         $\alpha_\psi$ & 2.0 & Walker-Walk-v0  \\
         $\alpha_\psi$ & 0.1 & Walker-Walk-v1  \\
         $\alpha_\psi$ & 1.0 & Cheetah-Run-v1  \\
         $\alpha_\psi$ & 1.0 & Sawyer-Peg-v0  \\
    \bottomrule
    \end{tabular}
}
\end{table}

\section{Additional Results}
We present additional results not in the main paper.

\subsection{Handling nonstationarity in the training environments}

In \cref{fig:cl-mujoco-all} we show performance on ZeUS and baselines on the training environments.

\begin{figure}[H]
\centering
\subfigure[Cheetah-Run-v0]{\label{fig:cl-halfcheetah-v2-eval-base-all}\includegraphics[width=0.23\textwidth]{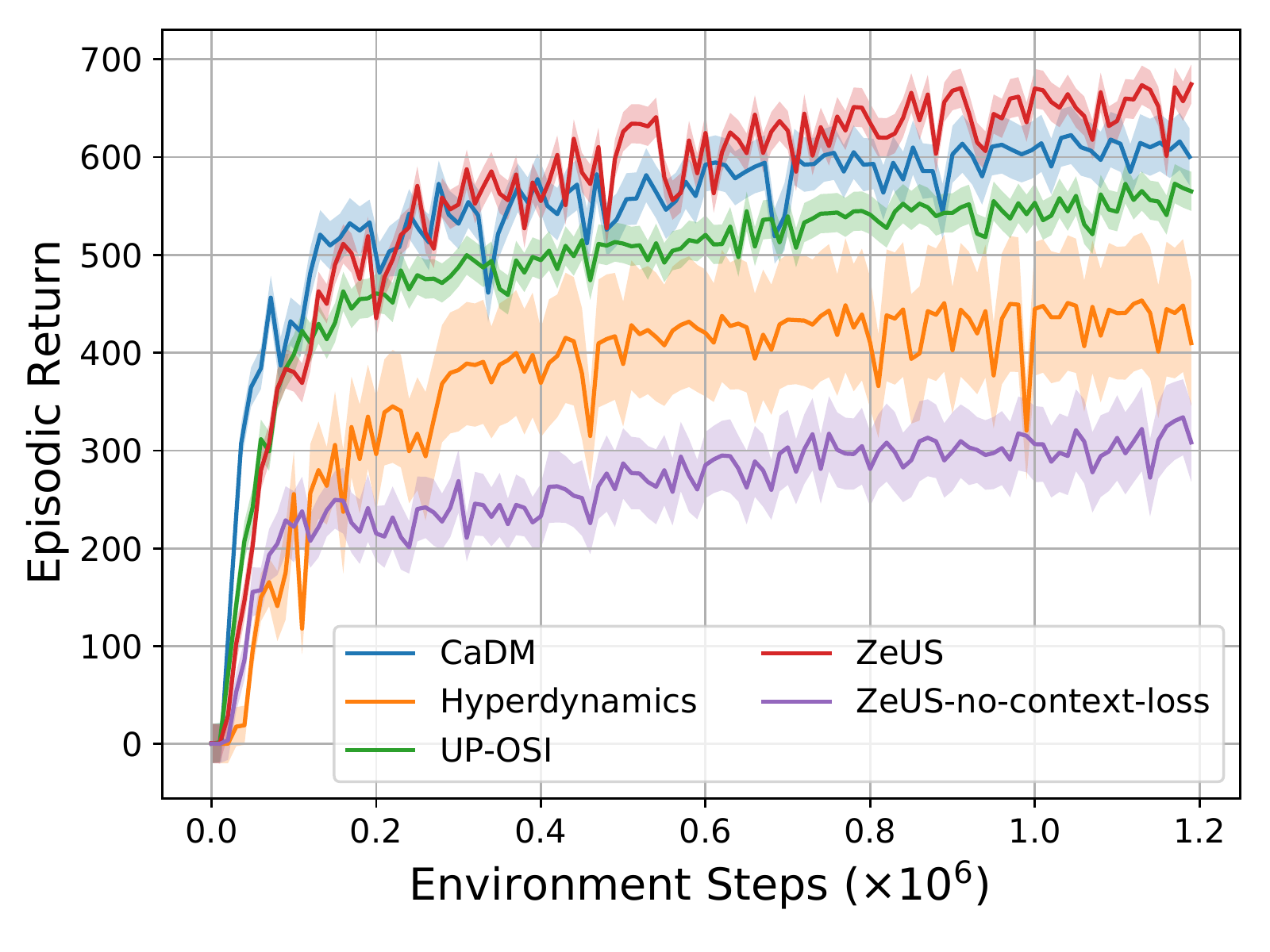}}
\subfigure[Finger-Spin-v0]{\label{fig:cl-finger-spin-v2-eval-base-all}\includegraphics[width=0.23\textwidth]{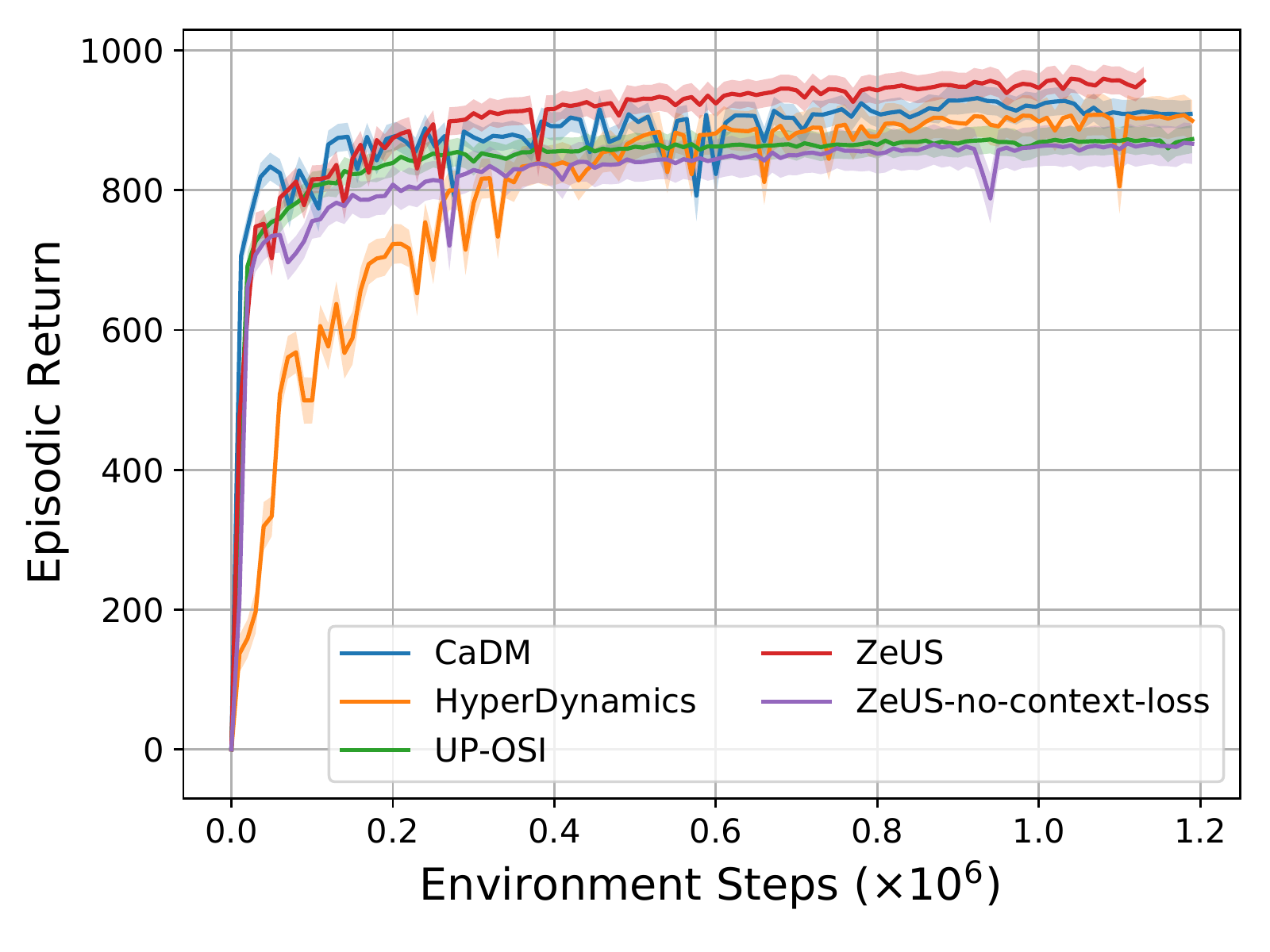}}
\subfigure[Walker-Walk-v0]{\label{fig:cl-walker-walk-v0-eval-base-all}\includegraphics[width=0.23\textwidth]{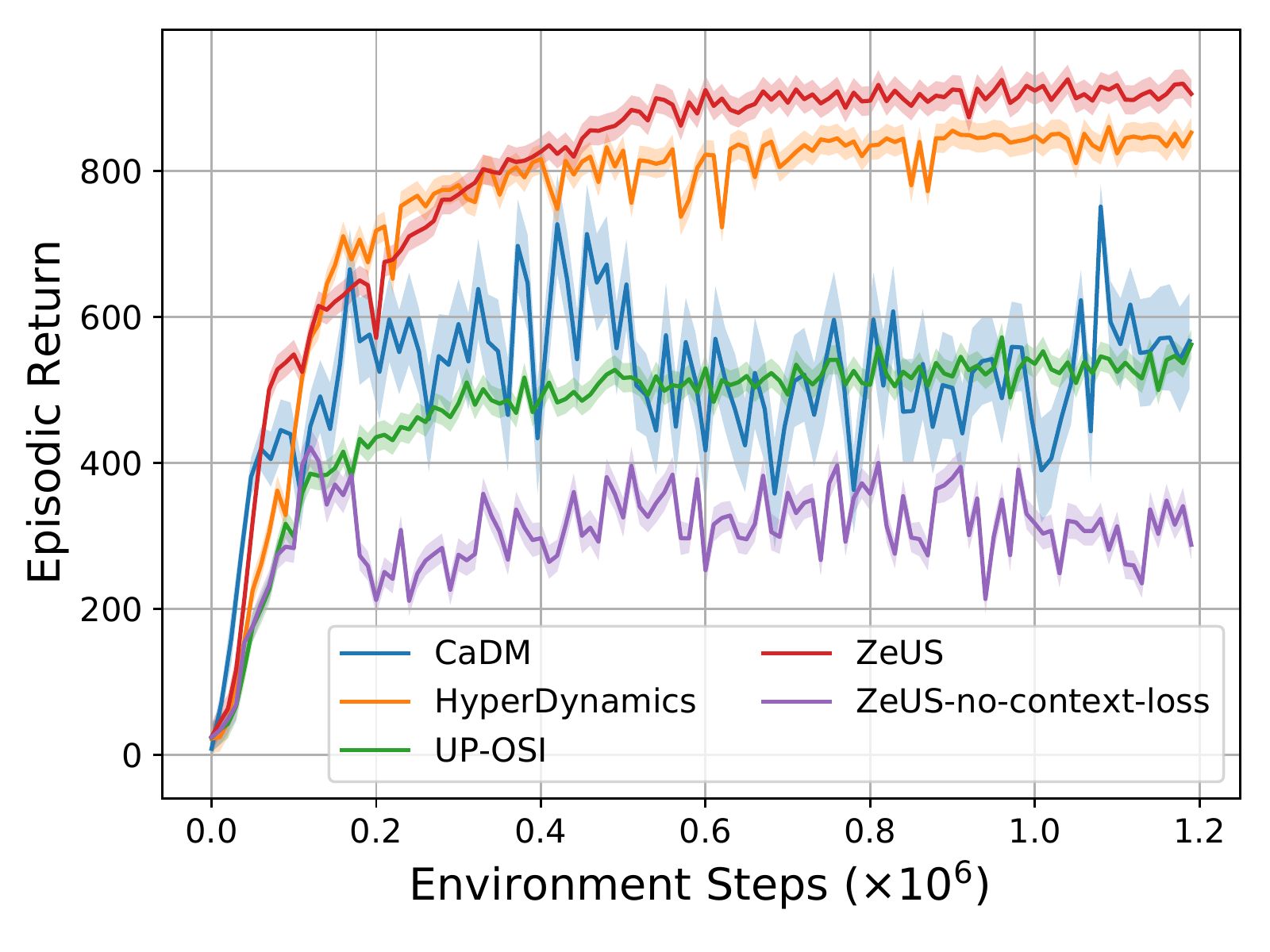}}
\subfigure[Walker-Walk-v1]{\label{fig:cl-walker-walk-v1-eval-base-all}\includegraphics[width=0.23\textwidth]{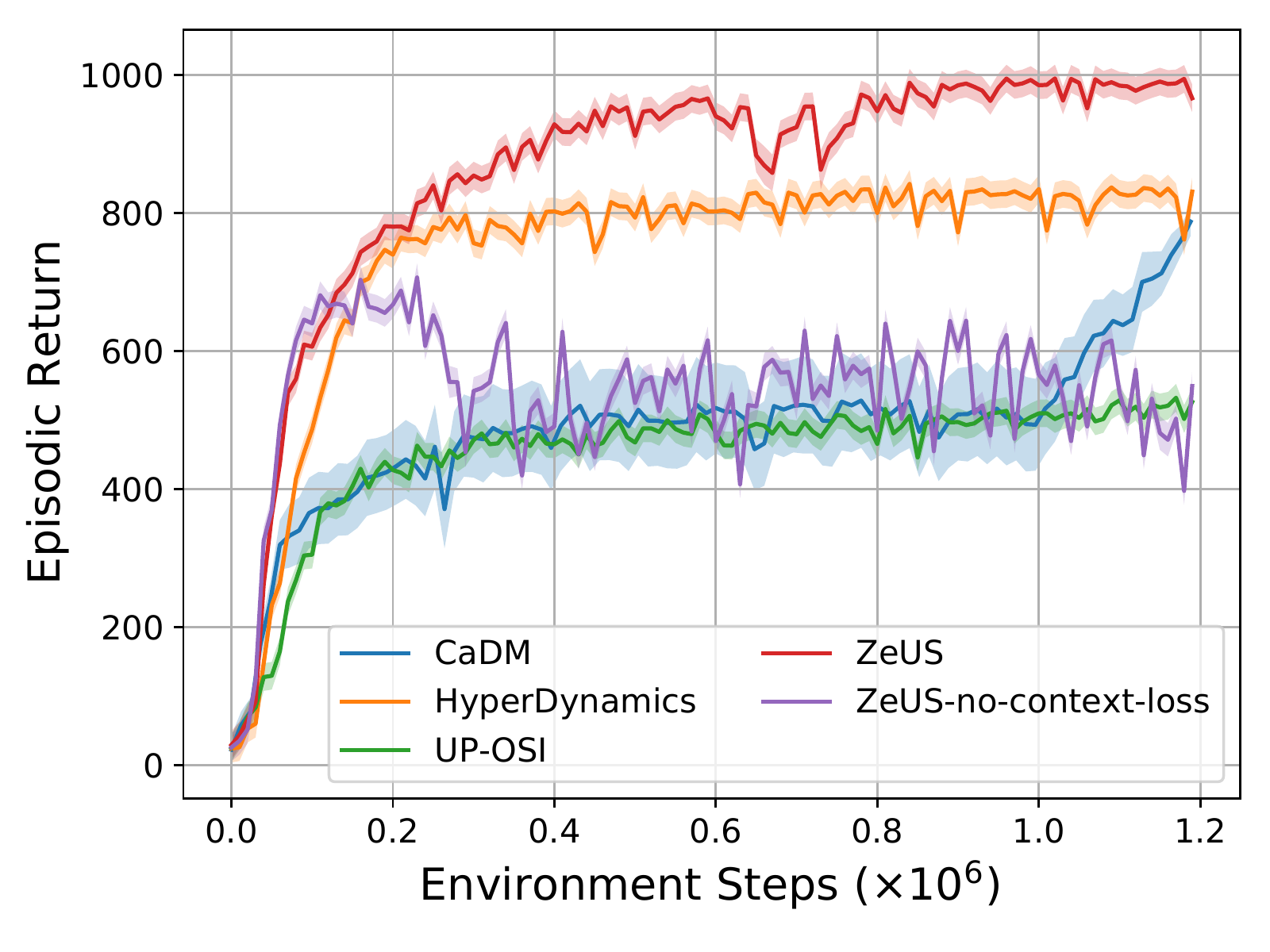}}
\caption{We compare the performance of the proposed ZeUS algorithm with~\textit{CaDM},~\textit{UP-OSI},~\textit{HyperDynamics} and~\textit{ZeUS-no-context-loss} algorithms on the training environments for four families of tasks with different dynamics parameters.}  
\label{fig:cl-mujoco-all}
\end{figure}

\subsection{Handling nonstationarity in the Interpolation environments}
In \cref{fig:cl-mujoco-all-interpolation} we show performance on ZeUS and baselines on evaluation environments that are interpolated from the train environments.
\begin{figure}[H]
\centering
\subfigure[Cheetah-Run-v0]{\label{fig:cl-halfcheetah-v2-eval-interpolation-all}\includegraphics[width=0.23\textwidth]{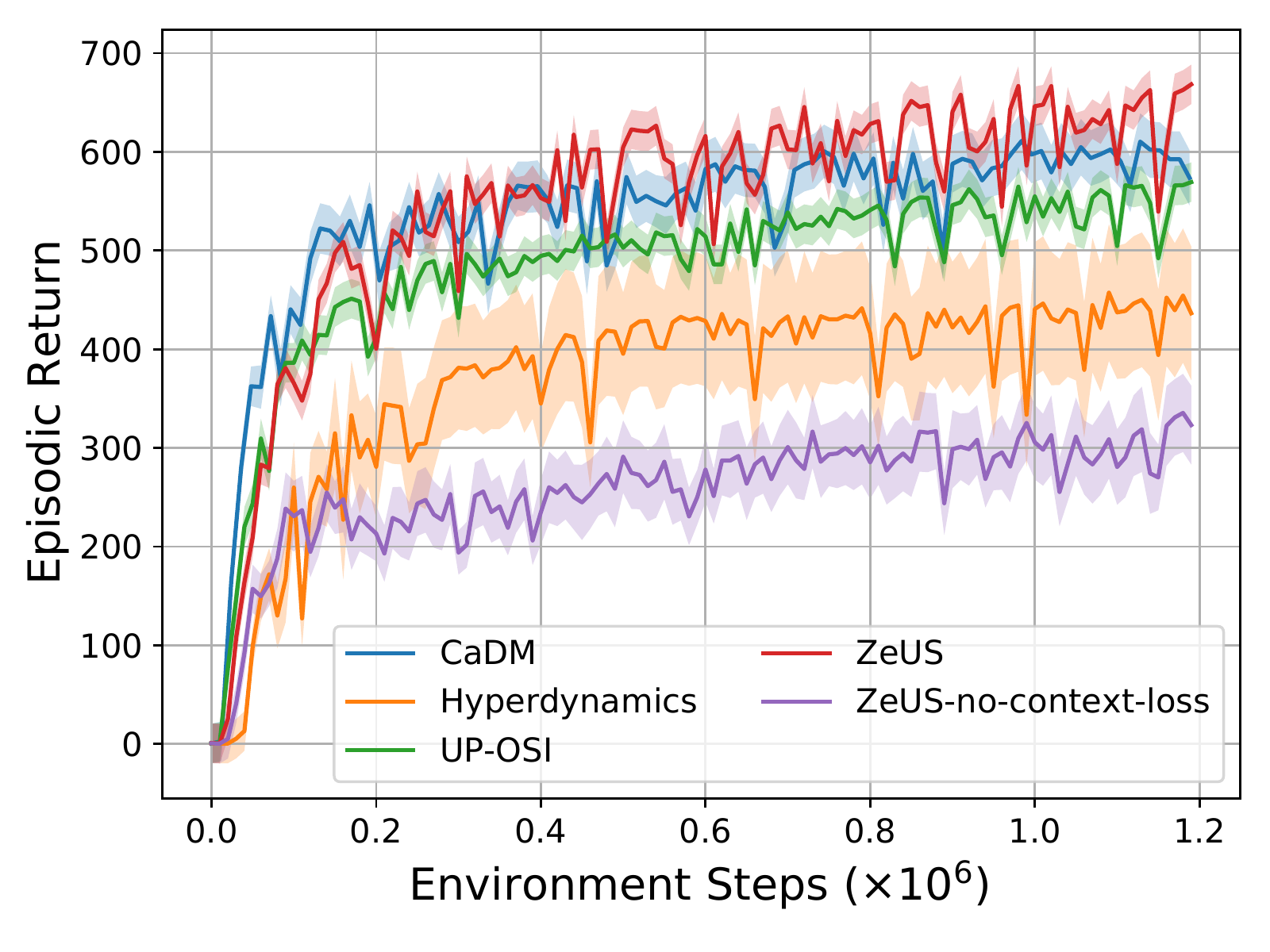}}
\subfigure[Finger-Spin-v0]{\label{fig:cl-finger-spin-v2-eval-interpolation-all}\includegraphics[width=0.23\textwidth]{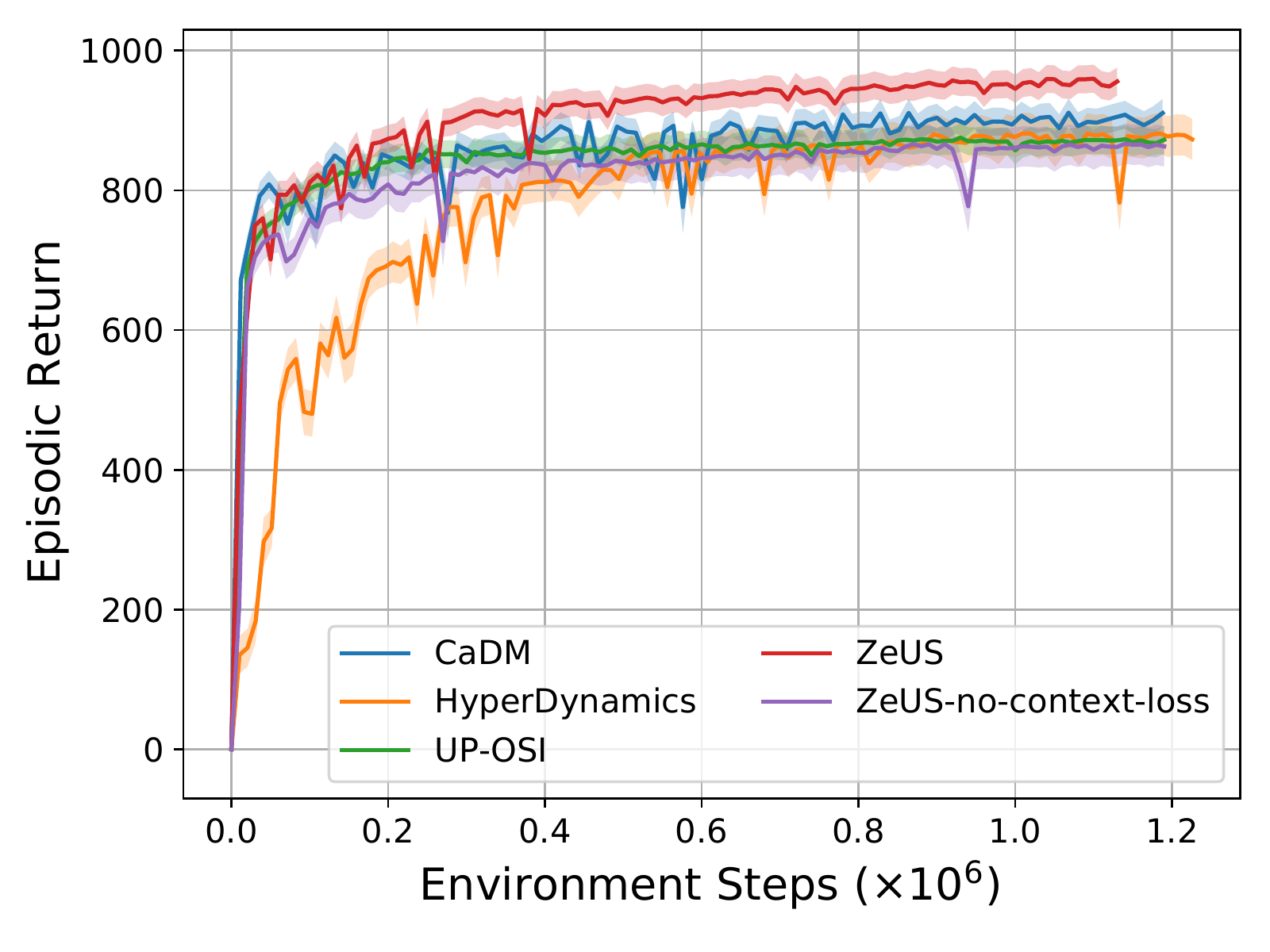}}
\subfigure[Walker-Walk-v0]{\label{fig:cl-walker-walk-v0-eval-interpolation-all}\includegraphics[width=0.23\textwidth]{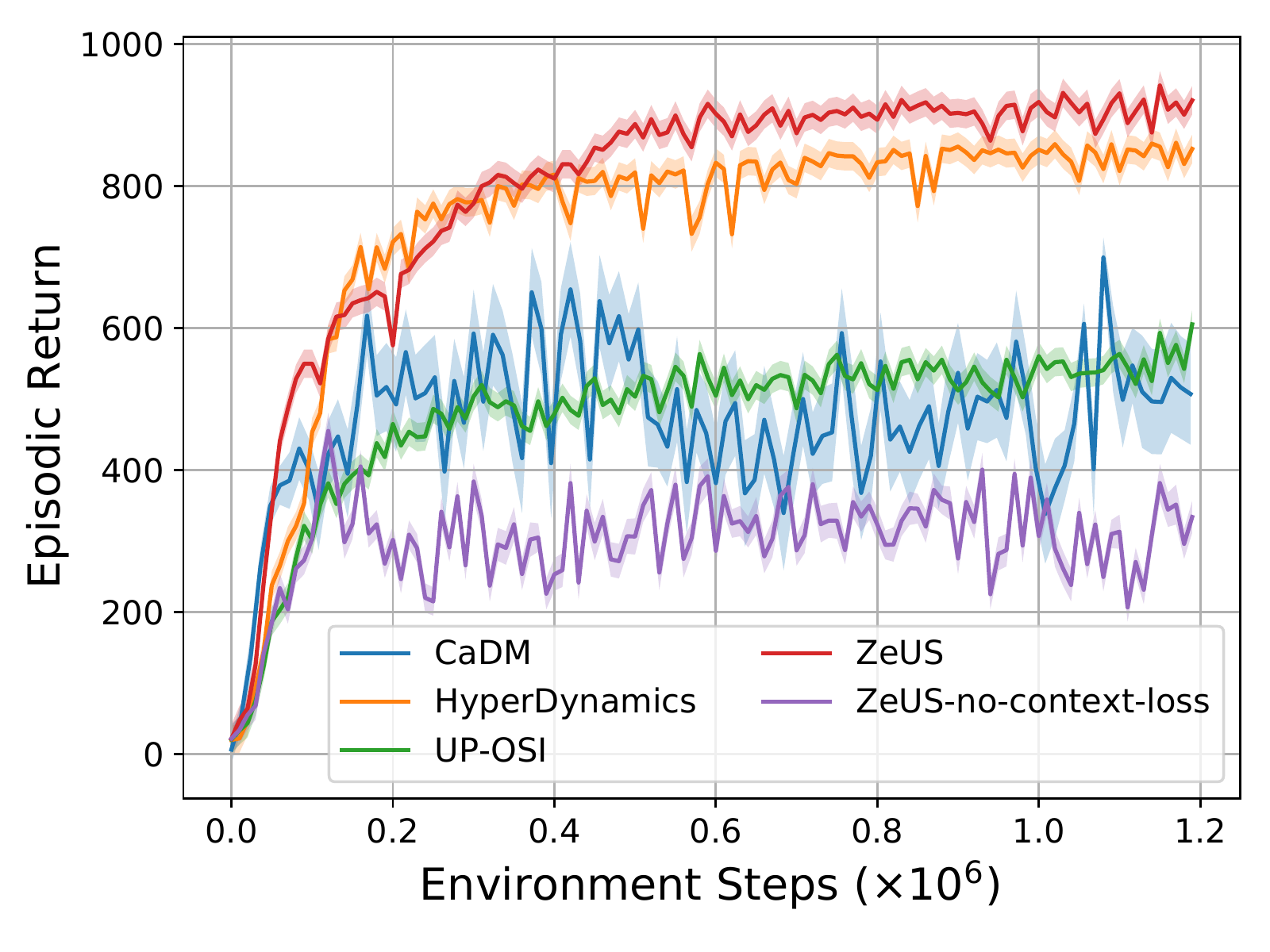}}
\subfigure[Walker-Walk-v1]{\label{fig:cl-walker-walk-v1-eval-interpolation-all}\includegraphics[width=0.23\textwidth]{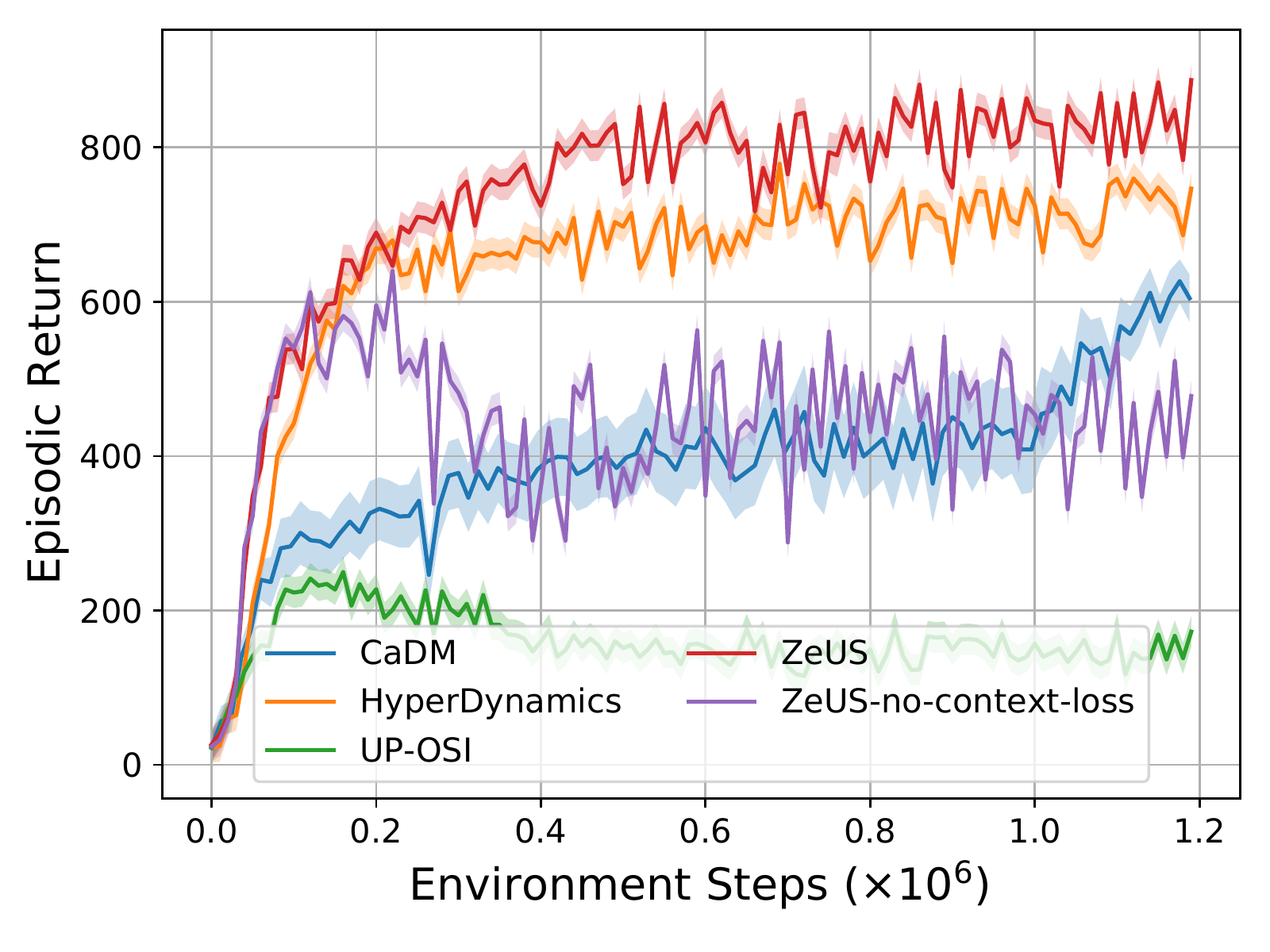}}
\caption{We compare the performance of the proposed ZeUS algorithm with~\textit{CaDM},~\textit{UP-OSI},~\textit{HyperDynamics} and~\textit{ZeUS-no-context-loss} algorithms on the evaluation environments (interpolation) for four families of tasks with different dynamics parameters.}  
\label{fig:cl-mujoco-all-interpolation}
\end{figure}

\subsection{Adapting and generalizing to environments with varying reward functions}
\label{app:adapting_and_generalizing_to_unseen_environments}

In~\cref{fig:cl-cheetah-v1-ablations} we show performance of ZeUS over a hyperparameter sweep for different values of aggregation operator and values of $\alpha_\psi$ for the Cheetah-run-v0 setup.

\begin{figure}[H]
\centering
\subfigure[Varying the aggregation operator in the context encoder]{\label{fig:cl-cheetah-v1-ablation-op}\includegraphics[width=0.48\textwidth]{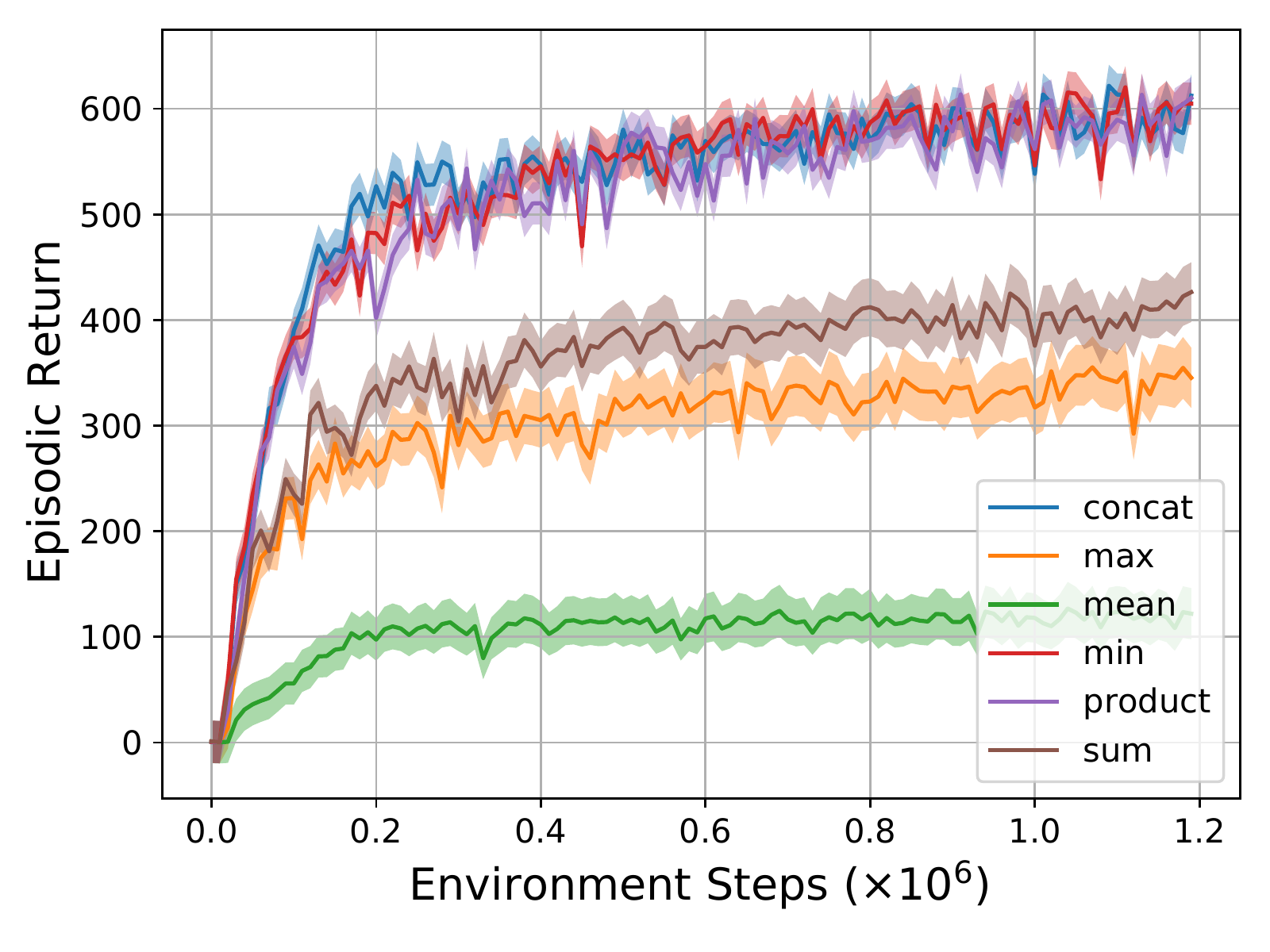}}
\subfigure[Varying values of $\alpha_\psi$]{\label{fig:cl-cheetah-v1-ablation-alpha}\includegraphics[width=0.48\textwidth]{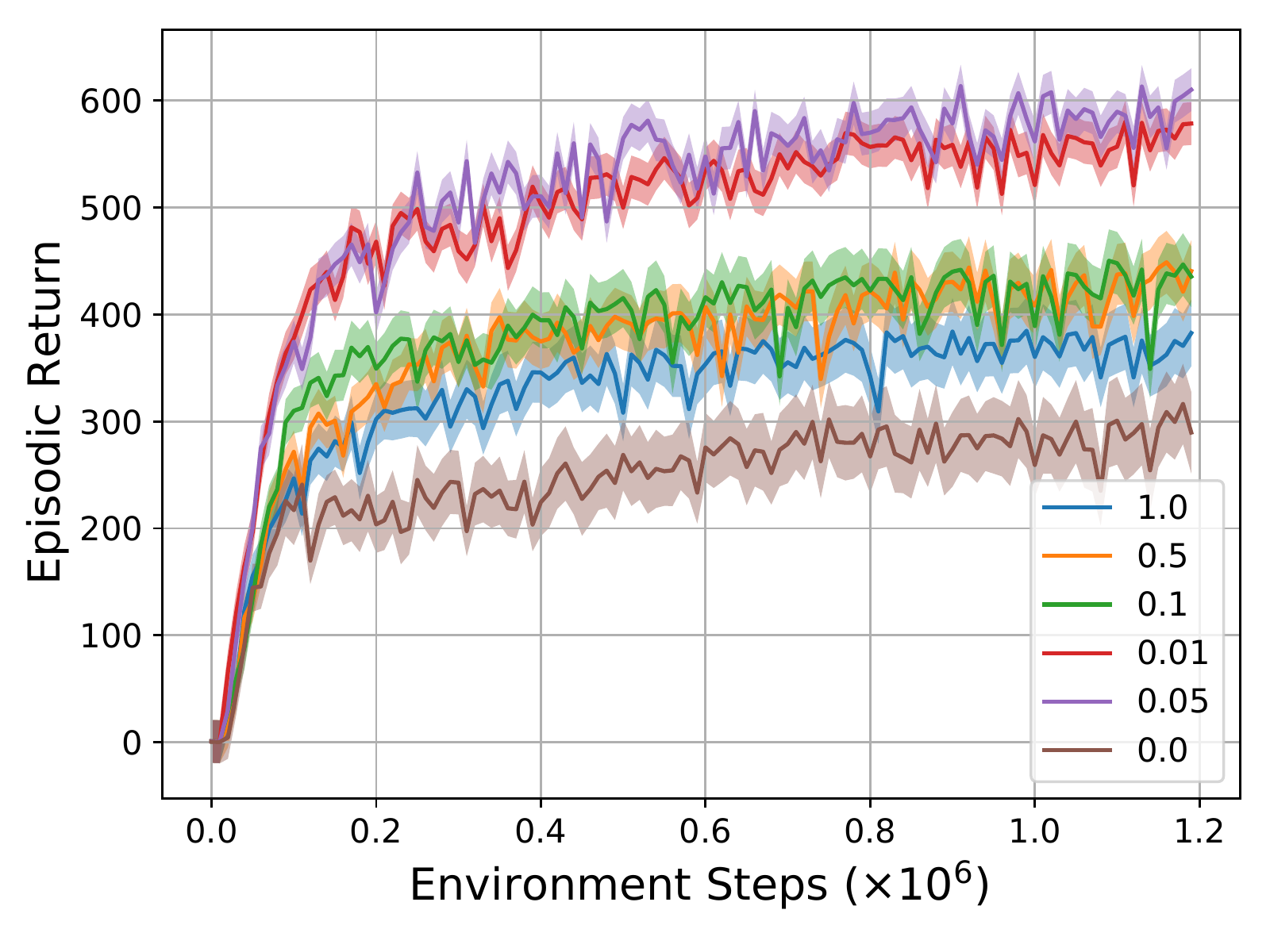}}
\caption{Performance of ZeUS on Cheetah-run-v0 task with varying values of the aggregation operator in the context enocder (in (a)) and varying values of $\alpha_\psi$ in (b).}  
\label{fig:cl-cheetah-v1-ablations}
\end{figure}

\section{Additional Theoretical Background}
\label{app:assumptions}
Bisimulation is a strict form of state abstraction, where two states are bisimilar if they are behaviorally equivalent. \textbf{Bisimulation metrics}~\citep{ferns2011contbisim} define a distance between states as follows: 
\begin{definition}[Bisimulation Metric (Theorem 2.6 in \citet{ferns2011contbisim})]
\label{def:bisim_metric}
Let $(\mathcal{S},\mathcal{A},T,r)$ be a finite MDP and $\mathfrak{met}$ the space of bounded pseudometrics on $\mathcal{S}$ equipped with the metric induced by the uniform norm. Define $F:\mathfrak{met} \mapsto \mathfrak{met}$ by
$$F(d)(s,s')=\max_{a\in \mathcal{A}}(|r(s,a) - r(s',a)| + \gamma W(d)(T_s^a,T_{s'}^a)),$$
where $W(d)$ is the Wasserstein distance between transition probability distributions.  Then $F$ has a unique fixed point $\tilde{d}$ which is the bisimulation metric.
\end{definition}
A nice property of this metric $\tilde{d}$ is that difference in optimal value between two states is bounded by their distance as defined by this metric. \cite{zhang2021dbc} use bisimulation metrics to learn a representation space that is Lipschitz with respect to the MDP dynamics in the single task setting.

Why is the bisimulation metric useful? It turns out that the optimal value function has the nice property of being smooth with respect to this metric.
\begin{lemma}[$V^*$ is Lipschitz with respect to $\tilde{d}$ \citep{ferns2004bisimulation}] 
\label{thm:lipschitz_bisim}
Let $V^*$ be the optimal value function for a given discount factor $\gamma$. Then $V^*$ is Lipschitz continuous with respect to $\tilde{d}$ with Lipschitz constant $\frac{1}{1-\gamma}$,
$$|V^*(s) - V^*(s')|\leq \frac{1}{1-\gamma}\tilde{d}(s,s').$$
\end{lemma}
Proof in \citep{ferns2004bisimulation}.
Therefore, we see that bisimulation metrics give us a Lipschitz value function with respect to $\tilde{d}$ with a Lipschitz constant $\frac{1}{1-\gamma}$. 

\section{Additional Theoretical Results and Proofs}
\label{app:thms}
\begin{theoremnum}[\ref{thm:lipschitz_bcmdp}]
Let $V^*$ be the optimal, universal value function for a given discount factor $\gamma$ and context space $\mathcal{C}$. Then $V^*$ is Lipschitz continuous with respect to $d_{\text{task}}$ with Lipschitz constant $\frac{1}{1-\gamma}$ for any $s\in\mathcal{S}$,
$$|V^*(s, c) - V^*(s, c')|\leq \frac{1}{1-\gamma}d_{\text{task}}(c,c').$$
\end{theoremnum}
\begin{proof}
We construct a super-MDP $\mathcal{M}'$ by concatenating context and state spaces into a new state space $\mathcal{S}':\mathcal{C}\times\mathcal{S}$. We can apply \cref{thm:lipschitz_bisim} to $\mathcal{M}'$ for any $s\in\mathcal{S},c,c'\in\mathcal{C}$:
\begin{equation}
\label{eq:lipschitz_ineq}
    |V_{\mathcal{M}'}^*((s,c)) - V_{\mathcal{M}'}^*((s,c'))|\leq \frac{1}{1-\gamma}\tilde{d}_{\mathcal{M}'}((s,c),(s,c')).
\end{equation}
By \cref{def:bisim_metric} and \cref{def:task_metric} we know that
\begin{equation*}
    \tilde{d}_{\mathcal{M}'}((s,c),(s,c')) \leq d_{\text{task}}(c,c').
\end{equation*}
So we can substitute the right hand side into \cref{eq:lipschitz_ineq},
\begin{equation*}
    |V_{\mathcal{M}'}^*((s,c)) - V_{\mathcal{M}'}^*((s,c'))|\leq \frac{1}{1-\gamma}d_{\text{task}}(c,c').
\end{equation*}
Which is equivalent to the following statement:
\begin{equation*}
    |V^*(s, c) - V^*(s, c')|\leq \frac{1}{1-\gamma}d_{\text{task}}(c,c').
\end{equation*}
\end{proof}

\subsection{Value and Transfer Bounds}
In this section, we provide value bounds and sample complexity analysis of the ZeUS approach. This analysis is similar to the one done in \cite{zhang2021hipbmdp}, which focused on a multi-environment setting with different dynamics but same task. 
We first define three additional error terms associated with learning a $\epsilon_R,\epsilon_T,\epsilon_c$-bisimulation abstraction,
\begin{equation*}
\begin{split}
\epsilon_R&:=\sup_{\substack{a\in\mathcal{A},\\o_1,o_2\in \mathcal{O},\phi(o_1)=\phi(o_2)}} \big|R(o_1,a)-R(o_2,a)\big|,  \\ \epsilon_T&:=\sup_{\substack{a\in\mathcal{A},\\o_1,o_2\in \mathcal{O},\phi(o_1)=\phi(o_2)}}\big\lVert \Phi T(o_1,a) - \Phi T(o_2,a) \big\rVert_1, \\
\epsilon_c&:=\|\hat{c} - c\|_1.
\end{split}
\end{equation*}
$\epsilon_R$ and $\epsilon_T$ are intra-context constants and $\epsilon_c$ is an inter-context constant.
$\Phi T$ denotes the \textit{lifted} version of $T$, where we take the next-step transition distribution from observation space $\mathcal{O}$ and lift it to latent space $\mathcal{S}$.
We can think of $\epsilon_R,\epsilon_T$ as describing a new MDP which is close -- but not necessarily the same, if $\epsilon_R,\epsilon_T>0$ -- to the original MDP. These two error terms can be computed empirically over all training environments and are therefore not task-specific. $\epsilon_c$, on the other hand, is measured as a per-task error. Similar methods are used in \citet{jiang2015abstraction} to bound the loss of a single abstraction, which we extend to the BC-MDP setting with a family of tasks.

\paragraph{Value Bounds.}
We first look at the single, fixed context setting, which can be thought of as the single-task version of the BC-MDP. We can compute approximate error bounds in this setting by denoting $\phi$ an $(\epsilon_R,\epsilon_T)$-approximate bisimulation abstraction, where
\begin{equation*}
\begin{split}
\epsilon_R&:=\sup_{\substack{a\in\mathcal{A},\\o_1,o_2\in \mathcal{O},\phi(o_1)=\phi(o_2)}}\big|R(o_1,a)-R(o_2,a)\big|, \\
\epsilon_T&:=\sup_{\substack{a\in\mathcal{A},\\o_1,o_2\in \mathcal{O},\phi(o_1)=\phi(o_2)}}\big\lVert \Phi T(o_1,a) - \Phi T(o_2,a) \big\rVert_1, \\
\epsilon_c&:=\|\hat{c} - c\|_1.
\end{split}
\end{equation*}
$\Phi T$ denotes the \textit{lifted} version of $T$, where we take the next-step transition distribution from observation space $\mathcal{O}$ and lift it to latent space $\mathcal{S}$.

\begin{lemma}
\label{thm:single_task_bisim}
Given an MDP $\bar{\mathcal{M}}$ built on a $(\epsilon_R,\epsilon_T)$-approximate bisimulation abstraction of Block MDP $\mathcal{M}$, we denote the  evaluation of the optimal $Q$ function of $\bar{\mathcal{M}}$ on $\mathcal{M}$ as $[Q^*_{\bar{\mathcal{M}}}]_{\mathcal{M}}$. 
The value difference with respect to the optimal $Q^*_\mathcal{M}$ is upper bounded by
\begin{equation*}
    \big\| Q^*_\mathcal{M} - [Q^*_{\bar{\mathcal{M}}}]_\mathcal{M}\big\|_\infty \leq \epsilon_R + \gamma \epsilon_T \frac{R_\text{max}}{2(1-\gamma)}.
\end{equation*}
\end{lemma}
\begin{proof}
From Lemma 3 in \citet{jiang2015abstraction}.
\end{proof}

We now evaluate how the error in $c$ prediction and the learned bisimulation representation affect the optimal $Q^*_{\bar{\mathcal{M}}_{\hat{c}}}$ of the learned MDP, by first bounding its distance from the optimal $Q^*$ of the true MDP for a single-task.
\begin{lemma}[$Q$ error]
\label{thm:qerror_bcmdp}
Given an MDP $\bar{\mathcal{M}}_{\hat{c}}$ built on a $(\epsilon_R,\epsilon_T,\epsilon_c)$-approximate bisimulation abstraction of an instance of a HiP-BMDP $\mathcal{M}_c$, we denote the  evaluation of the optimal $Q$ function of $\bar{\mathcal{M}}_{\hat{c}}$ on $\mathcal{M}$ as $[Q^*_{\bar{\mathcal{M}}_{\hat{c}}}]_{\mathcal{M}_c}$. 
The value difference with respect to the optimal $Q^*_\mathcal{M}$ is upper bounded by
\begin{equation*}
    \big\| Q^*_{\mathcal{M}_c} - [Q^*_{\bar{\mathcal{M}}_{\hat{c}}}]_{\mathcal{M}_c}\big\|_\infty \leq \epsilon_R + \gamma (\epsilon_T + \epsilon_c) \frac{R_\text{max}}{2(1-\gamma)}.
\end{equation*}
\end{lemma}
\begin{proof}
In the BC-MDP setting, we have a global encoder $\phi$ over all tasks, but the different transition distributions and reward functions condition on the context $c$. We now must incorporate difference in dynamics in $\epsilon_T$ and reward in $\epsilon_R$. Assuming we have two environments with hidden parameters $c_i,c_j$, we can compute $\epsilon^{c_i,c_j}_T$ and $\epsilon^{c_i,c_j}_R$ across those two environments by joining them into a super-MDP.:
For $\epsilon^{c_i,c_j}_T$:
\begin{equation*}
\begin{split}
    \epsilon^{c_i,c_j}_T&=\sup_{\substack{a\in\mathcal{A},\\o_1,o_2\in \mathcal{O},\phi(o_1)=\phi(o_2)}}\big\lVert \Phi T_{c_i}(o_1,a) - \Phi T_{c_j}(o_2,a) \big\rVert_1 \\
    &\leq \sup_{\substack{a\in\mathcal{A},\\o_1,o_2\in \mathcal{O},\phi(o_1)=\phi(o_2)}}\bigg(\big\lVert \Phi T_{c_i}(o_1,a) - \Phi T_{c_i}(o_2,a)\big\rVert_1 + \big\lVert \Phi T_{c_i}(o_2,a) - \Phi T_{c_j}(o_2,a) \big\rVert_1\bigg) \\
    &\leq \sup_{\substack{a\in\mathcal{A},\\o_1,o_2\in \mathcal{O},\phi(o_1)=\phi(o_2)}}\big\lVert \Phi T_{c_i}(o_1,a) - \Phi T_{c_i}(o_2,a)\big\rVert_1 + \sup_{\substack{a\in\mathcal{A},\\o_1,o_2\in \mathcal{O},\phi(o_1)=\phi(o_2)}}\big\lVert \Phi T_{c_i}(o_2,a) - \Phi T_{c_j}(o_2,a) \big\rVert_1 \\
\end{split}
\end{equation*}

For $\epsilon^{c_i,c_j}_R$ it is much the same:
\begin{equation*}
\begin{split}
    \epsilon^{c_i,c_j}_R&=\sup_{\substack{a\in\mathcal{A},\\o_1,o_2\in \mathcal{O},\phi(o_1)=\phi(o_2)}}\big|R_{c_i}(o_1,a)-R_{c_j}(o_2,a)\big| \\
    &\leq \sup_{\substack{a\in\mathcal{A},\\o_1,o_2\in \mathcal{O},\phi(o_1)=\phi(o_2)}}\bigg(\big| R_{c_i}(o_1,a) - R_{c_i}(o_2,a)\big| + \big| R_{c_i}(o_2,a) - R_{c_j}(o_2,a) \big|\bigg) \\
    &\leq \sup_{\substack{a\in\mathcal{A},\\o_1,o_2\in \mathcal{O},\phi(o_1)=\phi(o_2)}}\big| R_{c_i}(o_1,a) - R_{c_i}(o_2,a)\big| + \sup_{\substack{a\in\mathcal{A},\\o_1,o_2\in \mathcal{O},\phi(o_1)=\phi(o_2)}}\big| R_{c_i}(o_2,a) - R_{c_j}(o_2,a) \big| \\
\end{split}
\end{equation*}
    
Putting these together we get:
\begin{equation*}
    \epsilon^{c_i,c_j}_T + \epsilon^{c_i,c_j}_R \leq \epsilon_T + \epsilon_R + \|c_i - c_j\|_1
\end{equation*}

This result is intuitive in that with a shared encoder learning a per-task bisimulation relation, the distance between bisimilar states from another task depends on the change in transition distribution between those two tasks. We can now extend the single-task bisimulation bound (\cref{thm:single_task_bisim}) to the BC-BMDP setting by denoting approximation error of $c$ as $\|c - \hat{c}\|_1<\epsilon_c$.
\end{proof}

We can measure the generalization capability of a specific policy $\pi$ learned on one task to another, now taking into account error from the learned representation.
\begin{theoremnum}[\ref{thm:gen_bound}]
Given two MDPs $\mathcal{M}_{c_i}$ and $\mathcal{M}_{c_j}$, we can bound the difference in $Q^\pi$ between the two MDPs for a given policy $\pi$ learned under an $\epsilon_R,\epsilon_T,\epsilon_{c_i}$-approximate abstraction of $\mathcal{M}_{c_i}$ and applied to 
\begin{equation*}
\big\| Q^*_{\mathcal{M}_{c_j}} - [Q^*_{\bar{\mathcal{M}}_{\hat{c_i}}}]_{\mathcal{M}_{c_j}}\big\|_\infty \leq \epsilon_R + \gamma \big(\epsilon_T + \epsilon_{c_i} + \|c_i - c_j\|_1\big) \frac{R_\text{max}}{2(1-\gamma)}.
\end{equation*}
\end{theoremnum}
This result clearly follows directly from \cref{thm:qerror_bcmdp}. Given a policy learned for task $i$, \cref{thm:gen_bound} gives a bound on how far from optimal that policy is when applied to task $j$. Intuitively, the more similar in behavior tasks $i$ and $j$ are, as denoted by $\|c_i - c_j\|_1$, the better $\pi$ performs on task $j$.

\end{document}

%% file: math_commands.tex
\usepackage{amsmath,amsfonts,bm}

\def\eqref#1{equation~\ref{#1}}

\def\1{\bm{1}}

\DeclareMathAlphabet{\mathsfit}{\encodingdefault}{\sfdefault}{m}{sl}
\SetMathAlphabet{\mathsfit}{bold}{\encodingdefault}{\sfdefault}{bx}{n}